\documentclass[accepted]{uai2025} %
                        
\usepackage[american]{babel}

\usepackage{natbib} %
\usepackage{booktabs} %

\relax %
    \usepackage[dvipsnames]{xcolor}
    \usepackage{mathtools}
    \usepackage{amssymb}
		\DeclareMathSymbol{\shortminus}{\mathbin}{AMSa}{"39}
    \usepackage{bbm}
    \usepackage{faktor}
    \usepackage{graphicx}
    \usepackage{microtype}
    \usepackage{scalerel}
    \usepackage{enumitem}
    \usepackage{nicefrac}
    \usepackage[nobox]{restatelinks}

    \usepackage{stmaryrd}
    \usepackage{hyperref} %
        \hypersetup{colorlinks=true, linkcolor=blue!75!black, urlcolor=magenta, citecolor=green!50!black}

\relax %

	\DeclareMathOperator*{\Ex}{\mathbb{E}} %
	
    \newcommand{\ifrac}[2]{{#1}/{#2}}

    \newcommand{\mat}[1]{\mathbf{#1}}
    \DeclarePairedDelimiterX{\infdivx}[2]{(}{)}{%
		#1\;\delimsize\|\;#2%
	}

    \newcommand\pd[2]{\frac{\partial #1}{\partial #2}}

	\makeatletter
	\newcommand{\subalign}[1]{%
	  \vcenter{%
	    \Let@ \restore@math@cr \default@tag
	    \baselineskip\fontdimen10 \scriptfont\tw@
	    \advance\baselineskip\fontdimen12 \scriptfont\tw@
	    \lineskip\thr@@\fontdimen8 \scriptfont\thr@@
	    \lineskiplimit\lineskip
	    \ialign{\hfil$\m@th\scriptstyle##$&$\m@th\scriptstyle{}##$\hfil\crcr
	      #1\crcr
	    }%
	  }%
	}
	\makeatother

\relax %
    
    \newcommand\Lrn{\mathit{Lrn}}
    \newcommand\Bel{\mathit{Bel}}
    \newcommand\Plaus{\mathit{Plaus}}
    
    \newcommand\cseq{\ast}

    \newcommand{\ext}[1]{\overline #1} %

    \newcommand\confdom{[\bot,\!\top]}

    \newcommand\X{\mathcal X}

	\newcommand\Boltz{\mathrm{Boltz}}

\relax %
    \gdef\parencite{\citep}
    \gdef\textcite{\citet}

\usepackage{tikz}
	\usetikzlibrary{positioning,fit,calc, decorations, arrows, shapes, shapes.geometric}
	\usetikzlibrary{cd}

	\tikzset{AmpRep/.style={ampersand replacement=\&}}
	\tikzset{center base/.style={baseline={([yshift=-.8ex]current bounding box.center)}}}
	\tikzset{paperfig/.style={center base,scale=0.9, every node/.style={transform shape}}}

	\tikzset{dpadded/.style={rounded corners=2, inner sep=0.7em, draw, outer sep=0.3em, fill={black!50}, fill opacity=0.08, text opacity=1}}
	\tikzset{dpad0/.style={outer sep=0.05em, inner sep=0.3em, draw=gray!75, rounded corners=4, fill=black!08, fill opacity=1, align=center}}
	\tikzset{dpadinline/.style={outer sep=0.05em, inner sep=2.5pt, rounded corners=2.5pt, draw=gray!75, fill=black!08, fill opacity=1, align=center, font=\small}}

 	\tikzset{dpad/.style args={#1}{every matrix/.append style={nodes={dpadded, #1}}}}
	\tikzset{light pad/.style={outer sep=0.2em, inner sep=0.5em, draw=gray!50}}

	\tikzset{arr/.style={draw, ->, thick, shorten <=3pt, shorten >=3pt}}
	\tikzset{arr0/.style={draw, ->, thick, shorten <=0pt, shorten >=0pt}}
	\tikzset{arr1/.style={draw, ->, thick, shorten <=1pt, shorten >=1pt}}
	\tikzset{arr2/.style={draw, ->, thick, shorten <=2pt, shorten >=2pt}}

\usepackage{amsthm,thmtools} %
	\usepackage[noabbrev,nameinlink,capitalize]{cleveref}
    \theoremstyle{plain}
    \newtheorem{theorem}{Theorem}
	\newtheorem{coro}{Corollary}[theorem]
    \newtheorem{prop}[theorem]{Proposition}
    \newtheorem{fact}[theorem]{Fact}

    \newtheorem{lemma}[theorem]{Lemma}
    \theoremstyle{definition}
    \declaretheorem[name=Definition, qed=$\square$]{defn}
    \declaretheorem[name=Example, qed=$\square$]{example}

    \definecolor{openQcolor}{rgb}{0.9,0.2,0.9}

	\crefname{defn}{Definition}{Definitions}
	\crefname{prop}{Proposition}{Propositions}
    \crefname{issue}{Issue}{Issues}

\relax %

	\DeclareMathAlphabet{\mathdcal}{U}{dutchcal}{m}{n}
	\DeclareMathAlphabet{\mathbdcal}{U}{dutchcal}{b}{n}

    \newcommand{\nhphantom}[2]{\sbox0{\kern-2%
		\nulldelimiterspace$\left.\delimsize#1\vphantom{#2}\right.$}\hspace{-.97\wd0}}
	\makeatletter
	\newsavebox{\abcmycontentbox}
	\newcommand\DeclareDoubleDelim[5]{
	    \DeclarePairedDelimiterXPP{#1}[1]%
			{%
				\sbox{\abcmycontentbox}{\ensuremath{##1}}%
			}{#2}{#5}{}%
		    {%
				\nhphantom{#3}{\usebox\abcmycontentbox}%
				\hspace{1.2pt} \delimsize#3%
				\mathopen{}\usebox{\abcmycontentbox}\mathclose{}%
				\delimsize#4\hspace{1.2pt}%
				\nhphantom{#4}{\usebox\abcmycontentbox}%
			}%
	}
	\makeatother
	\DeclareDoubleDelim
		\SD\{\{\}\}
	\DeclareDoubleDelim
		\bbr[[]]
	\makeatletter
	\newsavebox{\aar@content}
	\newcommand\aar{\@ifstar\aar@one@star\aar@plain}
	\newcommand\aar@one@star{\@ifstar\aar@resize{\aar@plain*}}
	\newcommand\aar@resize[1]{\sbox{\aar@content}{#1}\scaleleftright[3.8ex]
		{\Biggl\langle\!\!\!\!\Biggl\langle}{\usebox{\aar@content}}
		{\Biggr\rangle\!\!\!\!\Biggr\rangle}}
	\DeclareDoubleDelim
		\aar@plain\langle\langle\rangle\rangle
	\makeatother

\relax %

\relax %
    \newcommand{\TODO}[1][INCOMPLETE]{{\centering\Large\color{red}$\langle$~\texttt{#1}~$\rangle$\par}}
    \newcommand{\dfootnote}[1]{%
        \let\oldthefootnote=\thefootnote%
		\setcounter{footnote}{999}
        \renewcommand{\thefootnote}{\textdagger}%
        \footnote{#1}%
        \let\thefootnote=\oldthefootnote%
    }

\definecolor{brownish}{rgb}{0.5, 0.2, 0.1}
\newtcolorbox{wip}{%
    colback=brownish!20!white,%
    title={$\langle$under construction$\rangle$},%
    enhanced jigsaw,
    breakable,
    colframe=brownish!40!white,%
}
\newtcolorbox{phaseout}{%
    colback={gray!02!white},
    coltext={gray!35!white},
    colframe={red!02!white},
    coltitle={red!35!white},
    title={~\hfill(depricated)},
    enhanced jigsaw,
    fontupper=\small,
    parbox=false,
    boxrule=0pt,
    sharp corners,
    breakable
}

\newtcolorbox{computation}{%
    colback={white},
    enhanced jigsaw,
    fontupper=\Large\itshape,
    fontlower=\small,
    parbox=false,
    boxrule=0pt,
    frame hidden,
    borderline west={4pt}{0pt}{green!20!black!40!white},
    sharp corners,
    breakable,
    enlarge left by=-4em,
    enlarge right by=4em,
    width=\linewidth+8em
}

\newtcbtheorem[use counter=example]{examplex}{Example}{
        label type=example,
        fonttitle=\bfseries,
        enhanced jigsaw,
        before=\par\medskip\noindent,
        parbox=false,
        colback=white,
        sharp corners,
        breakable,
    }
    {ex}

\makeatletter
\newcommand{\@minipagerestore}{\setlength{\parskip}{\medskipamount}}
\makeatother

\usepackage{xparse}
\let\realItem\item %
\newcommand\conflabel[1]{\textbf{[#1]}}

\makeatletter
\NewDocumentCommand\myItemboldperiod{o}{%
   \IfNoValueTF{#1}%
      {\realItem}%
      {\realItem[\conflabel{#1}]%
        \def\@currentlabel{#1}%
        \protected@edef\cref@currentlabel{[CFaxiomsi][][]#1}%
        }%
}
\makeatother

\newlist{CFaxioms}{enumerate}{1}
\setlist[CFaxioms]{
    resume,%
    label=\conflabel{CF\arabic{*}},
    ref={CF\arabic*},
    leftmargin=*,
    itemindent=1.5em,
    labelsep=1em,
    before=\let\item\myItemboldperiod,
    topsep=1ex
    }
\crefname{CFaxiomsi}{}{}
\crefrangeformat{CFaxiomsi}{#3#1#4#5--\crefstripprefix{#1}{#2}#6}
\newlist{LrnAxioms}{enumerate}{1}
\setlist[LrnAxioms]{
    resume,%
    label=\conflabel{L\arabic{*}},
    ref={L\arabic*},
    leftmargin=*,
    itemindent=1.5em, labelsep=1em, topsep=1ex,
    before=\let\item\myItemboldperiod,
    }
\crefname{LrnAxiomsi}{}{}
\crefrangeformat{LrnAxiomsi}{#3#1#4#5--\crefstripprefix{#1}{#2}#6}
\newlist{LrnBelAxioms}{enumerate}{1}
\setlist[LrnBelAxioms]{
    resume,%
    label=\conflabel{LB\arabic{*}},
    ref={LB\arabic*},
    leftmargin=*,
    itemindent=1.5em, labelsep=1em, topsep=1ex,
    before=\let\item\myItemboldperiod,
    }
\crefname{LrnBelAxiomsi}{}{}
\crefrangeformat{LrnBelAxiomsi}{#3#1#4#5--\crefstripprefix{#1}{#2}#6}
\newlist{URaxioms}{enumerate}{1}
\setlist[URaxioms]{
    resume,%
    label=\conflabel{UR\arabic{*}},
    ref={UR\arabic*},
    leftmargin=*,
    itemindent=1.5em,
    labelsep=1em,
    before=\let\item\myItemboldperiod,
    topsep=1ex}
\crefname{URaxiomsi}{}{}
\crefrangeformat{URaxiomsi}{#3#1#4#5--\crefstripprefix{#1}{#2}#6}
\newlist{CDaxioms}{enumerate}{1}
\setlist[CDaxioms]{
    resume,%
    label=\conflabel{CD\arabic{*}},
    ref={CD\arabic*},
    leftmargin=*,
    itemindent=1.5em,
    labelsep=0.5em,
    before=\let\item\myItemboldperiod,
    topsep=1ex}
\crefname{CDaxiomsi}{}{}
\crefrangeformat{CDaxiomsi}{#3#1#4#5--\crefstripprefix{#1}{#2}#6}

\newcommand\commentout[1]{}

  \marginprooflinksfalse

\newcommand\vnew\relax 

\title{Learning with Confidence}

\author[1,2]{\href{mailto:<oli@cs.cornell.edu>?Subject=Learning With Confidence}{Oliver E. Richardson}{}}

\affil[1]{%
    Computer Science Dept.\\
    Universit\'{e} de Montr\'eal\\
    Montr\'eal, Canada
}
\affil[2]{%
    Mila -- Quebec AI Institute
}

\begin{document}
\maketitle

\begin{abstract}

  We characterize a notion of confidence that arises in learning or updating beliefs:
    the amount of trust one has in incoming information and its impact on the belief state. 
  This \emph{learner's confidence} can be used alongside (and is easily mistaken for) probability or likelihood, but it is fundamentally a different concept---one that captures many familiar concepts in the literature, including learning rates and number of training epochs, Shafer's weight of evidence,  and Kalman gain.
  We formally axiomatize what it means to learn with confidence, give two canonical ways of measuring confidence on a continuum, and prove that confidence can always be represented in this way. 
  Under additional assumptions, we derive more compact representations of confidence-based learning in terms of vector fields and loss functions.
  These representations induce an extended language of compound ``parallel'' observations. 
  We characterize \emph{Bayesian} learning as the special case of an \emph{optimizing learner} whose loss representation is a linear expectation. 
\end{abstract}

\section{Introduction}\label{sec:intro}
\def\stmt{$A$}

\commentout{%
	The ability to articulate a \emph{degree of confidence}
	is a critical aspect of representing knowledge.
	There are
	many well-established ways to quantify (un)certinaty \parencite[\S2]{halpern2017reasoning},
		and chief among them is probability.
	While ``confidence'' can be coherently read in probabilistic terms,
		such usage may shadow another important concept.
	This paper details a different conception that arises when updating beliefs.
	As we shall see, this notion of confidence
	complements traditional representations of uncertainty (such as probability),
	and moreover unifies several different concepts across AI.
}

What does it mean to have a high degree of confidence in a statement $\phi$? 
It is often taken to mean that $\phi$ is likely.
We argue that there is also another conception of confidence that arises when learning---one that complements likelihood and, moreover, unifies several different concepts in the literature.
This kind of confidence is a measure of \emph{trust} in an observation $\phi$, rather than its likelihood;
it
quantifies how seriously to take $\phi$ in updating our beliefs.
So at one extreme,
if we observe $\phi$ but have no confidence in it,
we do not change our beliefs at all;
at the other, if we have full confidence in $\phi$,
we fully (and irreversibly) incorporate it into our beliefs.

\begin{example}
 \label{ex:prob-simple}
Suppose our belief state is a probability measure $P$, and we observe an event $\phi$.
The standard way to learn $\phi$ is to condition on it  (i.e., adopt belief state $P \mid \phi$). 
This is a full-confidence update; $\phi$ has probability 1 afterwards, 
	and conditioning on it again has no further effect.
Here is one obvious way
to interpret intermediate degrees of confidence:
starting with prior $P$ and
learning $\phi$ with confidence
$\alpha \in [0,1]$,
we end up with
posterior $(1-\alpha)P + \alpha (P \mid \phi)$.
Thus, having high confidence in $\phi$ leads to posterior beliefs that give $\phi$
high probability.
The converse is false, however.
If an untrusted source tells us $\phi$ which we already happen to believe,
then our prior assigns $\phi$ high probability,
we learn $\phi$ with low confidence,
and our posterior beliefs still give $\phi$ high probability.
\commentout{
Prior probability and confidence are further decoupled:
if we learn a surprising fact $\phi$ from a trusted source, we have high confidence in $\phi$ despite it having low prior probability.
}
\end{example}

Confidence allows us to be uncertain about observations,
which is quite different in principle from making observations that are uncertain.
\emph{Jeffrey's rule} (\citeyear{Jeffrey68}) 
is a well-established approach to the latter.
An important feature of the former, however, is that it enables
learning without fully committing to new observations.
Full-confidence updates, such as conditioning in \cref{ex:prob-simple}, are irreversible: 
from $\phi$ and the posterior $P|\phi$, it is not possible to recover the prior belief $P$.
The same is true of Jeffrey's rule,
which, in our conception, also prescribes full-confidence updates.
The concept we propose here is more similar to 
that behind of Shafer's \emph{Theory of Evidence} (\citeyear{shafer1976mathematical}),
although his account is specialized to a specific representation of uncertainty that has since fallen out of fashion.

\begin{example}
 	\label{ex:shafer}
Suppose our beliefs are represented by a 
\emph{%
	(Dempster-Shafer)
belief function},
which generalizes a probability measure over
a finite set $W$ of possible worlds.
\commentout{
    More precisely,
    we define our belief state to be a
    \emph{mass function} $m : 2^W \! \to\! [0,1]$
    satisfying $\sum_{U \subseteq W} m(U) \!=\! 1$
    and $m(\emptyset) \!=\! 0$.
    Such mass functions have a 1-1 correspondence with belief functions, and
    the belief function corresponding to $m$ is given by
    $\Bel_m(U) = \sum_{V \subseteq U} m(V)$
    \parencite{shafer1976mathematical}.}%
\def\complem#1{\overline{ #1}}%
Like a probability, a belief function $\Bel$ assigns to each event
$U \subseteq W$ a number $\Bel(U) \in [0,1]$,
with
$\Bel(\emptyset) = 0$ and $\Bel(W) = 1$.
It need not necessarily be that
$\Bel(U) + \Bel(\complem{U}) = 1$, but $\Bel$
must satisfy certain axioms (whose details do not matter for our purposes)
ensuring that
$
	\Bel(U) + \Bel(\complem{U}) \le 1.
$
$\Bel$ can be equivalently represented by its
\emph{plausibility function}
$\Plaus(U) := 1 - \Bel(\complem{U})$.
It is easy to see that $\Bel(U) \le \Plaus(U)$, and 
if $\Bel$ is a probability measure, then
$\Bel = \Plaus$.

\commentout{
	Suppose we come accross evidence that supports an event $\phi
	\subseteq W$ to a degree $\alpha \in [0,1]$.
	Together, $\phi$ and our confidence $\alpha$ in it
	can be represented by another mass function $s$,
	called a \emph{simple support function},
	by placing mass $\alpha$ on the event $\phi$, and the rest $(1-\alpha)$
	on the trivial event $W$.
	To combine our prior belief $m$ with the new evidence $s$,
	Shafer argues we should use Dempster's rule of combination
	to obtain a posterior $m' := m \oplus s$,
	which in this case equals:
	\begin{align*}
	 	m'(U) &=
		\frac{1}{\!\displaystyle 1 - \alpha \sum_{\mathclap{V \subseteq (W \setminus \phi)}} m(V)\!}
		\Big(
		(1-\alpha) m(U) +
		\alpha \sum_{\substack{\mathclap{V \subseteq W} \\ \mathllap{V \cap \phi} = \mathrlap{U}}} m(V)
			\Big).
	\end{align*}
	It is easy to verify that when $\alpha = 0$, the posterior beliefs are the same as the
	prior ones, and that when $\alpha = 1$,
	all mass is assigned to subsets of $\phi$.
	It follows that, after the update, $\Bel_{m'}(\phi)$.
	So again, we have two extremes in confidence, continuously parameterized
	by a value $\alpha \in [0,1]$.
	}
Suppose we come accross
evidence
that supports an event $\phi \subseteq W$
to a degree $\alpha \in [0,1]$.
Together, $\phi$ and our confidence $\alpha$ in it
can be represented by
the \emph{simple support function}
\vspace{-2ex}
\[
    \qquad\Bel_{(\alpha,\phi)}(U) := \begin{cases}
        1 & 
		\text{ if }U = W \\
        \alpha & \text{ if } \phi \subseteq U \subsetneq W \\
		0 &\text{ otherwise. } \\
    \end{cases}
\]

To combine belief functions,
	Shafer argues for Dempster's \emph{rule of combination} ($\oplus$). 
If we use $\oplus$ to combine two
simple support functions for $\phi$
with degrees of support $\alpha_1$ and $\alpha_2$, we get another simple support function
for $\phi$, with combined support $\alpha_1 + \alpha_2 - \alpha_1\alpha_2$.
As we will see \cref{sec:vecrep}, confidence also has an additive form. 
In Shafer's theory, this is the \emph{weight of evidence} $w = - k \log (1-\alpha)$ for some $k > 0$ [\citeauthor[pg 78]{shafer1976mathematical}].
The additive form of confidence plays a fundamental role in Shafer's theory,
	as it does in ours.

Using $\oplus$ to combine our prior with our evidence leads to
posterior belief $\Bel' := \Bel \oplus \Bel_{(\alpha,\phi)}$,
whose plausibility measure happens to be %
\begin{equation}
\Plaus'(U) = \frac
	{\alpha\; \Plaus(U \cap \phi) + (1-\alpha)\, \Plaus(U)}
	{1 - \alpha + \alpha\; \Plaus(\phi)}.
\label{eq:ds-plaus}
\end{equation}
It is easy to verify that
$\Bel' = \Bel$
when $\alpha = 0$,
and it can also be shown
that
$\Bel'(\phi) = \Plaus'(\phi) = 1$ when $\alpha = 1$.
So, as before, confidence $\alpha \in [0,1]$ parametrizes a continuous path
from ignoring $\phi$ to fully incorporating it.
\commentout{
	Alternatively, suppose that $m$ is not a probability but rather another simple support function on $\phi$. Then so is $m' = m\oplus s$.
	How much total evidence for $\phi$ does $m'$ represent?
	It is overwhelmingly standard to have a measurement that combines additively: if you had three (distinct) gallons of water and get another, you now have four; if you had six (independent) random bits and get three more, you now have nine.
	Is there an additive measure of confidence for simple support functions?
	Shafer calls such a quantity \emph{weight of evidence}, and proves that that of $s$ must be of the form $w = - k \log (1-\alpha)$ for some $k > 0$ [\citeauthor[pg 78]{shafer1976mathematical}].
	\commentout{
		Note that this is precisely the expression for $t$
		in \eqref{eq:loglogiota},
		because a choice of $\iota < 1$
		is equivalent to a choice of $k = \log(1-\iota) < 0$.
	}
	Weight of evidence
	is another important way of measuring confidence,
	and plays
	a fundemental rule in the theory of belief functions
	[\citeauthor[e.g.][Theorem 5.5]{shafer1976mathematical}]
}%
Yet the meaning of intermediate degrees of confidence can be subtle. 
In the special case where $\Bel = \Plaus$ is a probability
measure, 
a full confidence update ($\alpha=1$) yields the same conditioned 
probability $\Plaus' = (\Plaus | \phi)$ as in \cref{ex:prob-simple}.
Furthermore, the set of possible posteriors for intermediate $\alpha \in (0,1)$ is the same in both cases.
However, the two paths are parameterized differently;
	in fact, 
for all $\alpha  \in (0,1)$ the two updates disagree.
It follows that
the appropriate numerical value of $\alpha$ must depend on 
more than just an intuition of ``fraction of the way to the update''.
\commentout{
	We now look at some special cases. Suppose that $\Bel_m$ is a probability measure $\Pr$, or equivalently, that $m$ only assigns mass to singletons. Then $m'$ also only assigns mass to singletons, and is given by:
	\begin{equation}
		m'(\{x\}) =
	 	\frac{\alpha\; \Pr(\{x\} \cap \phi) + (1-\alpha)\, \Pr(\{x\})}{1 - \alpha + \alpha\; \Pr(\phi)}.
	 	\label{eq:ds-prob}
	\end{equation}
	Thus, as a function of $\alpha \in [0,1]$, $m'$ is a path that begins at $\Pr$,  ends at $(\Pr |\phi)$, and can even be viewed as a ``proportion of the way to incorporation'', just like in \cref{ex:prob-simple}---%
	yet intermediate values have different meanings.
	Therefore, to appropriately determine a numerical value of confidence, you need to know something more about the updating procedure.
	}%
\end{example}

Shafer's theory aims to address two seemingly problematic aspects of Bayesianism:
it  prescribes a belief representation that can better handle ignorance, 
and enables observations other than those that ``establish a single proposition with certainty'' \parencite[Chapter 1: \S7,\S8]{shafer1976mathematical}.
Ironically, in solving the first problem, his solution to the second becomes inaccessible to those who do not work with Dempster-Shafer belief functions. 
Our notion of learner's confidence directly addresses Shafer's second concern, but applies far more broadly.
A significant strength of our approach is that we do not take a stand on how beliefs should be represented---the concept of trust applies whether you use probability measures, belief functions, graphical models, imprecise probabilities, or something entirely different.
To illustrate, we 
now unpack the role of confidence in neural networks.

\begin{example}
		[Training a NN]\label{ex:classifier}
The ``belief state'' of a neural network may viewed as a setting $\theta \in \Theta \subseteq \mathbb R^d$ of weight parameters.
For definiteness, suppose we are talking about a classifier, so that
there is a space $X$ of inputs, a finite set $Y$ of labels,
and a parameterized family of functions
$\{ f_\theta : X \to \Delta Y \}_{\theta \in \Theta}$ mapping inputs $x \in X$ to distributions $f_\theta(x) \in \Delta Y$ over labels.
In the supervised setting, an observation is a pair $(x,y)$ consisting of an input $x$ 
labeled with class $y$.

Suppose we now observe $\phi = (x,y)$
with some degree of confidence;
how should we update the weights $\theta$?
\def\step{\mathtt{step}}
In contrast with previous examples, it is not so obvious
		how to learn
		 $\phi$ with full confidence.
Instead, modern learning algorithms
tend to be 
iterative
procedures
$\step: (X \times Y) \times \Theta \to \Theta$
that make small adjustments 
$\theta \mapsto \step(\phi,\theta)$
to the weights
\unskip. 
Each step is essentially a low-confidence update.
There is no guarantee, for example, that 
	$f_{\step(\phi,\theta)}(x)$ 
gives high probability to $y$---only that it is higher than $f_\theta(y|x)$.
This lower level of confidence is arguably what makes these learning algorithms robust to noisy and contradictory inputs. 
\commentout{
	In other words, such algorithms	do not take any one encounter with a training example too seriously.
	Indeed, this lower level of confidence
	is arguably what makes this learning process robust to noisy or contradictory inputs.
}\commentout{
	Modern learning algorithms (like gradient descent)
	are iterative procedures that
	make incremental changes to the weights.
	Therefore, if we perform one iteration of such a procedure 
	to update $\theta$ using a labeled training example $\phi = (x,y)$ to obtain new weights $\theta'$, there is no guarantee that $f_{\theta'}(x)$ gives high probability to $y$---only that it is higher than it was before.
	In other words, such algorithms 
		(in contrast to their historical counterparts like conjunction learning algorithms \parencite{conjunctions})
	do not take any one encounter with a training example too seriously---
	\unskip that is, they make low-confidence updates to the weights.
	This relative distrust of individual data points is arguably what makes the training process robust to noisy or contradictory observations.
}\commentout{
	As a result,
	there is a significant difference between going through the training data once
	\unskip, and doing so many times.}%

Higher confidence updates 
	can be obtained by 
	applying $\step$ more than once.
\def\thetainf{\theta_\infty}%
\def\thetalim{\theta_*}%
From initial weights $\theta_0$
	and defining $\theta_{n+1} = \step(\phi,\theta_{n})$,
	we get a sequence
	$(\theta_0, \theta_1, \theta_2, ...)$
	that converges to some $\thetalim \in \Theta$.
These limiting weights fully incorporate $\phi$ 
in the sense that
$\thetalim = \step(\phi,\thetalim)$, 
and also that 
$f_{\thetalim}(x)(y) = 1$ (at least if the network is sufficiently over-parameterized), i.e., $x$ is classified as $y$ with probability 1. 
Correspondingly, adopting belief $\thetalim$ is
appropriate only if we have complete trust in $\phi$,
meaning we find it critical that $x$ be classified as $y$.
(At the other extreme, 
	if we have no confidence in $\phi$, we should
	not update $\theta$ at all.)
Thus, the number of training iterations $n$
is a
measure of
confidence: it interpolates
between no confidence (zero iterations of $\step$) and full confidence
(infinitely many iterations of $\step$).
\commentout{
	This way of measuring confidence has a convenient property:
	first updating with confidence $n$ (that is, performing $n$ training iterations),
	and then afterwards updating with confidence $m$ (so $m$ additional iterations),
	is equivalent to a single update with confidence $m+n$.
	We call a measure of confidence that behaves this way \emph{additive}.
}%
Like Shafer's weight of evidence (\cref{ex:shafer}), the number of training iterations is an additive measure of confidence.

In the simplest settings, 
training examples do not come with confidence annotations,
in which case one effectively treats them all with 
	the same default confidence (by selecting a learning rate).
The number of times that $\phi = (x,y)$ appears in a dataset
	is then the de facto measure of confidence in $\phi$.
Often, though, these are not our intended confidences,
	which is why it can be helpful to remove duplicates 
	\citep{lee2021deduplicating}
	\unskip.
In richer settings, a more nuanced degree of confidence specific to each training example often arises, 
    such as agreement between annotators 
	\citep{artstein2017inter}
	\unskip,
    or
confidence scores in self-training \parencite{zou2019confidence}.

It is worth emphasizing that confidence
	is not always just a matter of accuracy.
Suppose, for example, that the classifier is intended to screen job applications, and that we want to make hiring practices less discriminatory.
In this case, we should have low confidence in training data based on prior hiring decisions---not because it is inaccurate, but because we do not trust it to inform our new hiring practice.
\end{example}

\commentout{
	\begin{figure}
	\centering
	\begin{tikzpicture}
		\begin{scope}[fill=gray,fill opacity=0.2,rounded corners=4px]
			\fill (0,0) rectangle (8,5); %
			\fill[] (0.2,0.1) rectangle (7.8,4.5); %
			\fill[] (0.4,0.2) rectangle (7.6,4.0); %
			\fill[] (0.6,0.3) rectangle (7.4,3.5); %
			\fill (0.8,0.4) -- (0.8, 3.0) -- (3, 3.0)
			 	to[out=0,in=0,looseness=2] (3,0.4) --cycle; %
			\fill (7.2,0.4) -- (7.2, 3.0) -- (5, 3.0)
			 	to[out=180,in=180,looseness=2] (5,0.4) --cycle; %
		\end{scope}
		\begin{scope}[anchor=north]
			\node at (4.0, 5.0) {Update Rules};
			\node at (4.0, 4.5) {Flow URs~~~$f$};
			\node at (4.0, 4.0) {Diffble URs~~~$X$};
			\node at (4.0, 3.5) {Conservative CFs~~~$\mathcal L$};
			\node at (2.5, 3.0) {Convex CFs};
			\node at (4.0, 2.5) {Linear CFs};
			\node at (5.5, 3.0) {Concave CFs};
		\end{scope}
	\end{tikzpicture}
	\caption{%
		A map of different kinds of commitment functions and their representations.}
	\end{figure}
	}

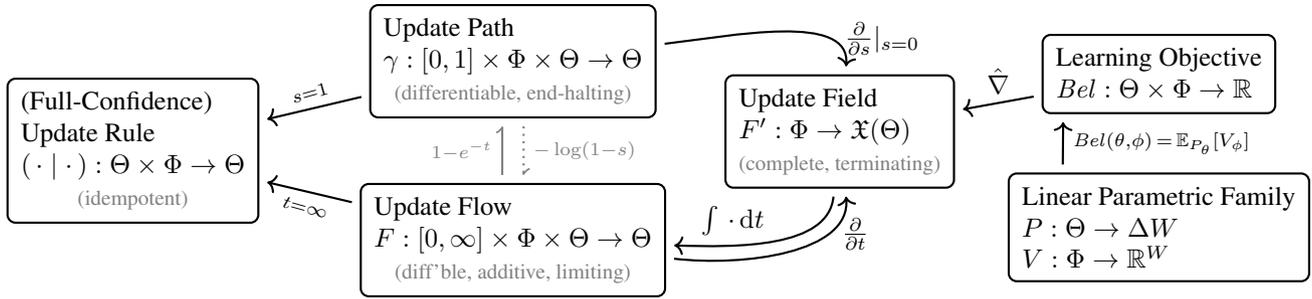
\begin{figure*}
\begin{center}
~\hfill
\hspace{-3cm}
\begin{tikzpicture}[yscale=0.85]
\begin{scope}[every node/.style={align=left,rounded corners=3,draw,thick,inner sep=5pt,anchor=center}]
	\node at (-5, 0) (fc) {%
		(Full-Confidence) \\
		Update Rule \\
		$(\,\cdot\mid \cdot\,) : \Theta \times \Phi \to \Theta$\\
			\smaller\hfill\color{gray}(idempotent)\hfill};
	\node at (0, -1.4)(flow) {%
		Update Flow\\
		$
		F
		: [0,\infty] \times \Phi \times \Theta \to \Theta$\\
			\smaller\hfill\color{gray}(diff'ble, additive, limiting)\hfill};
	\node at (0, 1.4)(path) {
		Update Path \\
		$\gamma: [0,1] \times \Phi \times \Theta \to \Theta$\\
		\smaller\hfill\color{gray}(differentiable, end-halting)\hfill};
	\node at (4.3, 0.3)(vfield) {%
		Update Field \\
		$F' : \Phi \to \mathfrak X(\Theta)$\\
		\smaller\hfill\color{gray}(complete, terminating)\hfill};
	\node at (8.5, 1.2) (loss) {%
		Learning Objective \\
		$\Bel: \Theta \times \Phi \to \mathbb R$};

	\node at (8.5, -1.2) (ev) {
		Linear Parametric Family  \\
		$P: \Theta \to \Delta W$ \\
		$V : \Phi \to \mathbb R^W$
	};
\end{scope}
	\draw[arr] (loss) to node[above]{$\hat\nabla$} (vfield);

	\draw[arr] (path) to[out=7,in=85]
		node[right=5pt,pos=0.7]{$\frac{\partial}{\partial s}|_{s=0}$} (vfield);
	\draw[arr] (flow) to[out=-7,in=-85]
		node[right=3pt,pos=0.7]{$\frac{\partial}{\partial t}$} (vfield);

	\draw[arr] (vfield) to[out=-95,in=-2] node[above=0pt,pos=0.70]
		{$\int \,\cdot\,\mathrm dt$}
		(flow);
	
	\draw[arr] (path) to node[above, sloped]{$\scriptstyle s=1$} (fc);
	\draw[arr] (flow) to node[below,sloped]{$\scriptstyle t=\infty$}(fc);
	
	\draw[arr] (ev.150) to node[right]{$\scriptstyle\Bel(\theta,\phi) \,=\, \Ex_{P_\theta}[V_\phi]$} (loss.-157);  %

	\draw[arr,-left to,dotted,gray] (path) edge[transform canvas={xshift=4pt}]
		node[right]{$\scriptstyle-\log(1-s)$} (flow);
	\draw[arr,-left to,gray] (flow) edge[transform canvas={xshift=-4pt}]
	 	node[left]{$\scriptstyle1-e^{-t}$}(path);
	
\end{tikzpicture}
\hspace{-3cm}\hfill~
\end{center}
\caption{\textit{Relationships between different representations of confidence-based learners.}
\small
Classical update rules like conditioning are projections (\cref{ssec:full-learn}), and correspond to learning with full-confidence (far left). 
\cref{theorem:add-reparam} guarantees an additive representation of a learner (middle left, bottom), which can also be represented with fractional confidence values (middle left, top), per \cref{prop:az-iso}.  Learners can be represented by vector fields (middle right), with benefits detailed in \cref{sec:vecrep}. We investigate the class of \emph{optimizing learners} induced by \cref{ax:lb-ascent} in \cref{sec:loss-repr} (far right, top), and the special case in which the learning objective is linear (far right, bottom). 
This turns out to characterize learning by application of Bayes Rule (\cref{prop:Boltz-Bayes}).
}\label{fig:map}
\end{figure*}

Perhaps the most important application of learner's confidence is  
in treating different sources of information with different degrees of trust.
Sensor fusion, which aims to combine readings from multiple sensors of various reliabilities, 
	is a clear example---%
	and Kalman filtering \citep{kalman1960new,brown1997introduction}, the standard approach to this problem, indeed comes with its own account of confidence.

\begin{example}[1D Kalman Filter]
	\label{ex:kalman1d}
\def\estx{\hat{x}}
Suppose we are modeling a
dynamical system whose state is a real number
$x \in \mathbb R$, and we receive
noisy measurements $z$ of $x$. 
The Kalman Filter 
tells us how to track this information 
with 
belief state $(\estx, \sigma^2)$,
where $\estx \in \mathbb R$ is our current estimate of
$x$, and
$\sigma^2$ is an uncertainty in that estimate, in the form of a variance. 
We now receive an observation
$z \sim \mathcal N(x, r^2)$ from a sensor.
How should we update our beliefs
\unskip?

The answer ranges from ignoring $z$ to replacing $\estx$ with it, depending on how much we trust the sensor.
The Kalman filter measures this trust with two (entangled) kinds of confidence: the precision $r^{-2}$ of the sensor, and a quantity $K$ called \emph{Kalman gain}.
The updated state
$(\estx', {\sigma^{2\prime}})$
is then:
\begin{align*}
	\estx' &= \estx + K (z - \estx)
	,
	&
	\sigma^{2\prime} &= (1 - K)^2 \sigma^2 + (K)^2 r^2.
\end{align*}
Like the other confidence measures we have seen, $K$
interpolates (linearly) between our prior mean $\estx$ and the new observation $z$, and (``quadratically'') between our prior uncertainty $\sigma^2$ and the sensor variance $r^2$.

More than in previous examples, we can also say something prescriptive about how to select a degree of confidence.
Assuming the goal is to maintain an unbiased estimate of $x$ with minimal uncertainty (as measured by expected squared error of $\estx$), 
and that $z$ is indeed the result of adding independent noise to $x$, 
then the optimal Kalman gain is
$
    K_{\mathrm{opt}}
        = \ifrac{\sigma^2}{(\sigma^2 + r^2)}
$
\parencite[p. 146]{brown1997introduction},
and $K$ is typically chosen this way in practice
\parencite{kalmanfilter.net}.
\commentout{
Plugging this value into the update equations, we find that this choice
makes $\estx'$ the average of our prior $\estx$ and new observation $z$,
weighted by their respective variances.
}%
Let us now revisit the extremes. 
If $K = 0$, which
    is optimal
when $z$ has unbounded variance,
    the belief state remains unchanged:
intuitively, there is so much noise in observations that
    we ignore them.
At the other extreme, if no noise is added ($r^2=0$),
then $K_{\mathrm{opt}} = 1$ and we end up with a posterior $(z, 0)$
based solely on the new observation.
\end{example}

\Cref{ex:kalman1d} features three kinds of (un)certainty:
\begin{enumerate}[left=0.1em,nosep,parsep=\parskip]
\item \textbf{Learner's Confidence:} a subjective trust 
	in how seriously to take an observation for updating (e.g., $K$)%
        .
    \label{item:learn-conf}

\item \textbf{Internal (Epistemic) Confidence:}
        the degree of uncertainty present in a given belief state,
        either overall ($\sigma^2$)
        or in a given statement
        (e.g., the density 
			$
			\phi \mapsto 
			\mathcal N(\phi|\hat x, \sigma^2)$).
	Internal confidences in our other examples
        include the probability $\Pr(\phi)$ in
        \cref{ex:prob-simple}, the degree
        of belief $\Bel(\phi)$ in \cref{ex:shafer},
        and the value of the loss function $\mathcal L(\theta,\phi)$
        used to train the classifier in \cref{ex:classifier}.
	\label{item:epistemic-conf}

\item
    \textbf{Statistical (Aleatoric) Confidence:}
    an objective measure of the (un)reliablility of an observation,
    based on historical data and/or modeling assumptions about how
    observations arise
    (e.g., the noise level $r^2$)%
    .
    \label{item:stat-conf}
\end{enumerate}
The three senses of the word ``confidence'' are related,
    but different in nature.
A great deal of work has already gone into understanding the differences between senses \ref{item:epistemic-conf} and \ref{item:stat-conf} \citep{der2009aleatory,hullermeier2021aleatoric}.
We (obviously) focus on sense \ref{item:learn-conf},
which we have tried to distinguish from
more pervasive usage of the word (sense \ref{item:epistemic-conf})
to quantify subjective likelihood, degree of belief, or (un)certainty.
Nevertheless,
    epistemic confidences (sense \ref{item:epistemic-conf}) may be thought of as aggregate
    reflections of learner's confidence (sense \ref{item:learn-conf}) in past observations;
	conversely, it is often possible to define learner's confidence by its effect on epistemic confidence
	(see \cref{sec:loss-repr}).

One should also distinguish  
	learner's confidence (sense \ref{item:learn-conf}),
	at least in principle,
    from statistical confidences
    (sense \ref{item:stat-conf})
    such as the variance in readings
    of a sensor (\cref{ex:kalman1d})
    or annotator agreement (\cref{ex:classifier}).
When available,
    the statistical reliability of an information source
    should absolutely play a role in determining how seriously we take
    it in updating our beliefs;
learner's confidence informed exclusively by a probabilistic model can be seen as an important (``aleatoric'') special case of our theory.
Still, statistical confidence presupposes that observations are drawn (independently) from a (fixed) distribution, while learners's confidence is meaningful even without such assumptions.
\commentout{%
We may not always know the variances of our sensors, and that such a quantity is well-defined is a significant assumption on its own.
Statistical confidences typically require us to know that observations are drawn independently from a fixed distribution, while learners's confidence can be meaningful even without this assumption.}

\textbf{Contributions.}
We hope that these examples have given the reader an intuitive sense of what confidence is, how ubiquitously it arises, and why it is important.
In the remainder of the paper, we study confidence more formally,
	making a series of successively stronger assumptions
	(all satisfied by \cref{ex:prob-simple,ex:shafer,ex:classifier,ex:kalman1d}).
Each set of assumptions enables a new more compact representation for
	a learning rule, summarized in \cref{fig:map}.
In \cref{sec:formalism}, we develop a formal framework 
	laying out axioms for our notion of confidence.
In \cref{sec:conf-continuum}, we focus on the properties of confidence in a continuum, developing 
	vector-field and loss-based representations of learners. 
This can enable simultaneous orderless updates, even in settings where it was not previously possible. 
Finally, we analyze Bayesian updating in \cref{sec:Bayes}.

\section{A Formal Model of Confidence, Learning, and Belief}
	\label{sec:formalism}

Our formalism consists of three components: 
	a domain $\confdom$ of confidence values,
	a space $\Theta$ of belief states, and
	a language $\Phi$ of possible observations. 
For instance:
\begin{itemize}[nosep,itemsep=1pt,left=0.5em]
    \item In \cref{ex:prob-simple}, $\Theta$ is the set of probability
    measures on some measurable space $(\Omega, \mathcal F)$,
    $\Phi$ is the $\sigma$-algebra $\cal F$, and the confidence domain
    is $[0,1]$.
    \item In \cref{ex:shafer}, $\Theta$ is the set of belief functions
    over a finite set $W$, $\Phi = 2^W$ is the set of subsets of $W$,
    and confidence is a degree of support $\alpha \in [0,1]$
    or a weight of evidence $w \in [0,\infty]$.
    \item In \cref{ex:classifier}, 
	$\Theta \subseteq \bar{\mathbb R}^d$ is
    the space of network parameters, $\Phi = X \times Y$ is the space of 
    input-lablel pairs, and the confidence domain is the
	 	extended natural numbers $\{0, 1,\ldots, \infty\}$ under addition.
	\item In \cref{ex:kalman1d}, $\Theta = \Phi = \mathbb R$, 
		The domain of $K$ is $[0,1]$, and the domain of $\sigma^2$ is $[\infty, 0]$. 
		Together, the pair $(K, \sigma^2)$ acts as a measure of confidence. 
\end{itemize}
We call $(\Theta, \Phi, \confdom)$ a \emph{learning setting}.
In this setting, a \emph{learner}
is a function
\commentout{\unskip\footnote{%
	It should be straightforward to extend our theory
	so as to handle randomized updates as well;
	the point is that
	the belief state, observation, and confidence must together contain
	enough information to describe the updating process.
}}
$
	\Lrn : \Phi \times \confdom \times \Theta \to \Theta
$
that describes the belief update process.
Explicitly: from a prior belief $\theta$, and a statement $\phi$
	observed with some degree of confidence $\chi$,
	a learner 
	produces a posterior belief state
	$\Lrn(\phi,\chi,\theta) \in \Theta$.
We use superscripts and subscripts to
fix some arguments
of $\Lrn$ and view it as a function of the others. So
$\Lrn(\phi,\chi,\theta)$
can equivalently be written as
$
	\Lrn_\phi(\chi,\theta) = \Lrn^\chi_\phi(\theta)
	= \Lrn^\chi(\phi,\theta) = \Lrn_{(\theta,\phi)}(\chi)
		 .
$
The rest of \cref{sec:formalism} develops axioms for $\Lrn$ and supporting concepts intended to capture intuitions about learning.

We proceed in three stages.
After starting with an abstract theory of confidence domains $\confdom$ themselves
	(\cref{ssec:confdom}),
we then axiomatize confidence-based updates to beliefs in $\Theta$ (\cref{ssec:comm-func}).
Finally, we bring in observations $\Phi$ 
	and the function $\Lrn$
	 (\cref{ssec:full-learn}).

\subsection{Abstract Confidence Domains}
	\label{ssec:confdom}
A \emph{confidence domain} $(D, \le, \bot, \top, \cseq, \mathfrak g)$
is a set $D$
of confidence values 
equipped with a preorder $\le$,
a 
least element
$\bot$ (``no confidence''), a greatest element
$\top$ (``full confidence''),
and an operation $\cseq$ that combines two 
independent 
degrees of confidence.
We often abbreviate a confidence domain as $D = \confdom$,
leaving $\le$ and $\cseq$ implicit.
\commentout{%
Because $\cseq$ represents \emph{independent} combination,
	we require that it be commutative and associative. 
}%
We want to ignore independent information we have no confidence in, and, if already fully confident, remain so in the face of new independent information. 
Formally, this amounts to requiring, for all $\chi,\chi',\chi'' \in D$:
\def\rightdescmargin{0.2cm}
\begin{itemize}[parsep=0pt,itemsep=1pt,label={},left=\rightdescmargin]
\item $(\chi \cseq \chi') \cseq \chi'' = \chi \cseq (\chi' \cseq \chi'')$
    \hfill (associativity)$\mathrlap{,}$
		\hspace{\rightdescmargin}\;\;
\item $\bot \cseq \chi = \chi$
    \hfill (that $\bot$ is neutral)$\mathrlap{,}$
		\hspace{\rightdescmargin}\;\;
\item $\top \cseq \chi = \top$\;
    \hfill (and that $\top$ is absorbing)$\mathrlap{.}$
		\hspace{\rightdescmargin}\;\;
\end{itemize}
Finally, $D$ comes with geometric information $\mathfrak g$, which may 
	include topology or differentiable structure.
We are especially interested in two continuous domains
	from our examples.
The first is the \emph{fractional domain} $[0,1]$,
whose elements $s \in [0,1]$ represent 
	the ``proportion of the way towards complete trust''.
If you go proportion $s$ towards fully trusting something,
then $s'$ of the remaining way, then overall
you have gone
$s \cseq s'
:= s + s' (1-s) = 
s + s' - s \cdot s'$
    of the way to complete trust.
The other confidence domain of particular interest
 	is the \emph{additive domain} $[0,\infty]$, which is
	ideal for analogies of time and weight.
	
\begin{linked}{prop}{az-iso}
	The fractional domain $[0,1]$ and the additive domain $[0,\infty]$ are isomorphic.
	Furthermore, the space of isomorphisms between them 
		is
		in natural bijection with $(0,\infty)$.
	Specifically, for each $\beta \in (0,\infty)$, there is
		an isomorphism $\varphi_\beta : [0,1] \to [0,\infty]$ given by
	$\varphi_\beta(s) = -\frac1\beta \log(1-s)$
	with inverse $\varphi_\beta^{-1}(t) = 1- e^{-\beta t}$.
\end{linked}

The fact that these two domains are equivalent but only up to $\beta$---a ``choice of units'' in the additive domain, or ``tempering'' in the fractional domain----%
implies that many standard ways of quantifying confidence are equivalent, yet also highlights the fundamental difficulty of doing so in absolute terms (as we began to see at the end of \cref{ex:shafer}).

\commentout{%
Keep in mind that there are confidence domains as well.
The interval $[0,1]$ with $\cseq = \max$ is an important one that is not isomorphic to the additive or fractional domains. 
}%
There are also confidence domains that are not isomorphic to the additive or fractional domains. 
The interval $[0,1]$ with $\cseq = \max$ is an important 1-dimensional one; 
confidence domains can also be multi-dimensional or discrete, but 
our results in \cref{sec:conf-continuum,sec:Bayes} say little about these cases.

	\subsection{Belief States and Commitment Functions}
		\label{ssec:comm-func}
		\label{sec:learning-setting-funcs}
We now reintroduce belief states $\theta \in \Theta$
in order to describe the role of confidence in belief updating.
Observations $\phi$ come later (\cref{ssec:full-learn});
we find that the most essential aspects of confidence can already be understood through the behavior of a function $F = \Lrn_\phi : \confdom \times \Theta \to \Theta$ that describes the learning process for some fixed and abstract $\phi$.
We call such a function $F$ a \emph{commitment function} if it obeys the axioms in this subsection (\cref{ax:zero,ax:combinativity,ax:cont-and-smooth,ax:acyclic,ax:seq-for-more}) intended to ensure that $F$ respects the structure of the confidence domain.

\textbf{No Confidence.}
Having no confidence ($\chi=\bot$) in an observation $\phi$ should lead us to ignore it. 

\begin{LrnAxioms}
    \item
		$
		\forall \phi,\theta.\quad
		\Lrn_\phi^\bot(\theta) = 
		\Lrn_\phi(\bot, \theta) = \theta
		$.
        \label{ax:zero}
\end{LrnAxioms}

\textbf{Full-confidence.}
Since the purpose of
$\Lrn^\top_\phi$
is to \emph{fully} incorporate $\phi$ into our beliefs,
two successive full-confidence updates with the same information ought to have the same effect as a single one:
having fully integrated $\phi$ into our beliefs, 
there is nothing to do upon observing $\phi$ again.

\commentout{
\begin{defn}
	A \emph{full-confidence update rule} is
	a mapping $P: \Phi \times \Theta \to \Theta$ such that
	for all $\phi \in \Phi$, 
	$P_\phi = (\theta \mapsto P(\phi,\theta)): \Theta \to \Theta$ is idempotent.
	That is,	
	$P_\phi(P_\phi(\theta)) = P_\phi(\theta)$
	 for all $\phi\in\Phi$ and $\theta \in \Theta$.
\end{defn}}

\begin{LrnAxioms}
	\item[FC]
	Full-confidence updates are idempotent.
    That is, for all $\phi \in \Phi$,  $\Lrn^\top_\phi \circ \Lrn^\top_\phi = \Lrn^\top_\phi$.
	\label{ax:idemp}
\end{LrnAxioms}

Once $\Theta$, $\Phi$, and any relevant relationships between them are specified, there is often a natural choice of full-confidence update rule.
We illustrate with three examples. 
In each case, the possible belief states $\Theta := \Delta W$ be the set of all probability distributions over a finite set
 $W
  $ of possible worlds.

\begin{enumerate}[wide, label=\textit{(\arabic*)},itemsep=0.05ex,topsep=0pt,labelindent={1em}]
	\item %
	\textbf{Conditioning.}
	First, consider the case where observations are events, i.e., $\Phi := 2^W$.
	The overwhelmingly standard way to update is to condition: 
	starting with $P \in \Delta W$, the conditional measure 
	$P|A \in \Delta W$ is given by $(\mu|A)(B) = \ifrac{P(B \cap A)}{P(A)}$, provided $P(A) > 0$.
	Note that $(P|A)|A = P|A$, so the update is idempotent.
	\commentout{%
	There are well-known issues with conditioning $\mu$ on $A$ when
	$P(A) = 0$, 
	and so typically this operation is left undefined. 
	}%

	\item
	\textbf{Imaging }\parencite{lewis1976probabilities}\textbf{.}
	Suppose
 	we already have a full-confidence update rule
	$f : \Phi \times W \to W$
	that, 
	 given $\phi \in \Phi$ and $w \in W$, produces the world $f(\phi, w) \in W$ ``most similar to $w$, in which $\phi$ is true'' \parencite{gardenfors1979imaging}.
	Idempotence of $f_\phi: W \to W$
	means the world most similar to $f(\phi,w)$ in which $\phi$ is true, is $f(\phi,w)$ itself.
	We can then 
	lift $f$ to a full confidence update rule for $\Delta W$,
	by
	$%
    		F(\phi, P) 
				(A) := 
				P(\{w : f(w, \phi){ \in} A\})
	$,
	intuitively moving the mass of $w$ to
	$f(\phi,w)$.	
	Since $f$ is idempotent, so is $F$.

	\commentout{
	\item More generally, consider a measurable space $\mathcal W = (W, \mathcal A)$, where $W$ is a set and $\mathcal A$ is a $\sigma$-algebra over $W$, and let $\mathcal F \subset \mathcal A$ be closed under supersets in $\mathcal A$.

	\TODO[Properly Use Conditional Probability Measure, to define on all events]

	Conditioning a probability distribution $\mu \in \Delta\X$ on an event $A \in \mathcal A$ also makes sense in this more general measure-theoretic setting, at least so long as $\mu(A) > 0$, and is given by
	$$
		(\mu \mid A) (B) = \frac{\mu(B \cap A)}{\mu(A)}
	$$
	}

	\item
	\textbf{Jeffrey's Rule.}
	The two previous approaches to updating establish that an event with certainty.
	Jeffrey's rule ($\mathit J$) addresses this limitation
		by allowing for uncertain (i.e., probabilistic) observations.
	Formally, let $\Phi$ be the set of pairs $(X,\pi)$
	where $X : W {\to} S$ is a random variable taking values in a set $S$,
	and $\pi \in \Delta S$ is a probability on
	$S$.
	Jeffrey's update rule is:
	$
		{J}((X,\pi),
			P) := \sum_{x \in S} \pi(X{=}x)  P \big|
            (X{=}x).
	$
	When $\pi$ places all mass on some $x \in S$, $\mathit J$ conditions on $X {=} x$.
	For this reason, $\mathit J$ is thought to
		generalize conditioning 
		to observations of ``lower confidence''.
	Yet even when $\pi$ is not deterministic, $J$ \emph{fully} incorporates
	$\pi$ into the posterior beliefs:
	the marginal of $J((X,\pi),P)$ on $X$ is $\pi(X)$,
	and the prior belief 
	$P(X)$ has been destroyed.
	Indeed, $J_{(X,\pi)}$ is idempotent. 
	Therefore, $J$ still establishes observations with full confidence---%
		it's just that those observations are probabilities.
	Experience suggests that this point can be counter-intuitive; we submit that the confusion is clarified by a conception of confidence distinct from likelihood.
\end{enumerate}

\vnew{%

}

\cref{ax:idemp} implies that full-confidence updates are not invertible: they destroy information in the prior, often making for a simpler posterior. 
This potential simplification of future calculations is a major benefit of fully trusting information.
However, full-confidence updates are extreme.
An agent that updates by conditioning, for instance,
permanently commits to believing everything it ever learns
(and thus gains nothing from making the same observation again later). 
Clearly humans are not like this; revisiting information
 	helps us learn \parencite{ausubel1965effect}.
Similarly, artificial neural networks are trained with
 	many incremental updates, and benefit from seeing 
	the training data many times.
We would like an account that allows for less extreme belief alterations,
in which information is only partially incorporated.
This is the role of intermediate degrees of confidence.

\textbf{Geometry.}
Learner's confidence interpolates between 
	ignoring new information and fully defering to it,
	and we would like that interpolation to be continuous and differentiable.
	
\begin{LrnAxioms}
	\item
	If $\confdom$ and $\Theta$ are both topological spaces, then 
	for all $\theta$ and $\phi$, 
	the map
	$
	\Lrn_{(\theta,\phi)} = 
	\chi \mapsto 
	\Lrn(\theta,\chi,\phi)
	$
	is continuous.
	If $\confdom$ and $\Theta$ are both manifolds, then 
	$\Lrn_{(\theta,\phi)}$ is 
	differentiable.
	Furthermore, $\Lrn_\phi^\chi$ is differentiable on a subset $\Theta_\phi$ defined in \cref{prop:maximal-continuous-theta} below.%
		\label{ax:cont-and-smooth}
\end{LrnAxioms}

Ideally the posterior would be continuous in our prior beliefs as well as $\chi$.
This suggests
a simpler strengthening of \cref{ax:cont-and-smooth}:
that
$\Lrn_\phi$ 
be continuous (and differentiable) as a function of $(\chi,\theta)$%
---yet this is often too much to ask for.

\begin{linked}{prop}{no-continuous-condition-ext}
	Take $\Theta = \Delta W$ and $\phi \subseteq W$.
	There exists no continuous function $\Lrn_\phi : \Delta W \times [0,1] \to \Delta W$ 
	with the property that 	$\Lrn_\phi(\mu, 1) = \mu|\phi$ when $\mu(\phi) > 0$. 
\end{linked}

This result is yet another perspective on the familiar difficulties with conditioning on events of probability zero,
but intuitively this should be an edge case.
Instead of imposing an axiom, 
	we observe that it is possible to 
	capture the phenomenon in a useful way even at this abstract level. 

\begin{linked}{prop}{maximal-continuous-theta}
	For all $\phi \in \Phi$, 
	there is a maximal open set 
	$\Theta_\phi \subseteq \Theta$ such that
	the restriction
	$
	\Lrn_{\phi} |_{\Theta_\phi} : 
		[\bot,\!\top) \times \Theta_\phi \to \Theta
	$		
	of 
	$\Lrn_\phi$
	to $\Theta_\phi$ is continuous. 	
\end{linked}
In our examples, $\Theta_\phi$ consists of those
belief states that do not flatly contradict $\phi$.
In \cref{ex:prob-simple}, \cref{prop:no-continuous-condition-ext,prop:maximal-continuous-theta}
imply that $\Theta_\phi = \{ \mu \in \Delta W : \mu(\phi) > 0\}$
is the set of distributions for which conditioning on $\phi$ is defined.
In \cref{ex:classifier}, $\Theta_{(x,y)}$ is the set of parameters at which gradients $\nabla_{\theta}\ell(f_\theta(x), y)$ of the loss $\ell$ are finite.

\textbf{Order.}
For a learner, the defining feature of
	the ordering $\chi <  \chi'$ is that
learning with higher confidence ($\chi'$) can done by first 
making the more conservative, lower-confidence ($\chi$) update, followed by a nontrivial residual update.

\begin{LrnAxioms}[nosep]
	\item 
	 $\exists \delta : \{ (\chi', \chi) : \chi' \ge \chi \} \times \Theta \to \confdom$ continous such that $\Lrn_\phi(\delta(\chi',\chi, \theta),\Lrn_\phi(\chi,\theta)) = \Lrn_\phi(\chi',\theta)$
	 and $\delta(\chi,\chi,\theta) = \bot$.
	 \label{ax:ineq-witness}
	 \label{ax:seq-for-more}
	\commentout{%
	\item
	$\forall \theta,\chi,\chi'.\quad$
	$\chi < \chi'$ 
        $\quad\implies$ \\
        \phantom{a}$\quad
		 \exists \chi''\!.\, \bot {<} \chi'' {\le} \chi' \text{ and}$
		$\Lrn_\phi^{\chi''} \!{\circ}\, \Lrn_\phi^\chi (\theta) = \Lrn_\phi^{\chi'}\!(\theta)$.
        \label{ax:ineq-witness}
		\label{ax:seq-for-more}
	}%
	\commentout{%
	\item[\cref*{ax:seq-for-more}$^{<}$]
	$\forall \phi, \theta,\chi,\chi'.\quad$
	$\chi < \chi'$ 
        $\quad\iff\quad$
        $\exists \chi'' < \chi'.~~$
		$\Lrn_\phi^{\chi''} \circ \Lrn_\phi^\chi (\theta) = \Lrn_\phi^{\chi'}(\theta)$.
		\label{ax:ineq-witness-strict}
		\label{ax:seq-for-more-strict}
	}%
\end{LrnAxioms}

Furthermore, learning is not cyclic: if learning with confidences $\chi_0$ and $\chi_1$ have the same effect, then the same is true of all confidences $\chi_0 \le \chi \le \chi_1$ between them.

\begin{LrnAxioms}
	\item If $\chi_0 \le \chi \le \chi_1$ and $\Lrn_\phi(\chi_0, \theta) = \Lrn_\phi(\chi_1, \theta)$, then $\Lrn_\phi(\chi,\theta) = \Lrn_\phi(\chi_0, \theta)$.
		\label{ax:acyclic}
\end{LrnAxioms}

\textbf{Independent Combination.}
$\Lrn$ should be used to incorporate information to the extent that it is novel,
i.e., information that is not already accounted for in our prior beliefs.
Thus, we would like a sequence of two independent
observations in the same observation $\phi$ to
be equivalent to a single observation of $\phi$ 
with their combined degree of confidence.

\begin{LrnAxioms}[nosep]
	\item 
	$\!\forall \phi, \chi,\chi'.~
	\Lrn_\phi(\chi, \Lrn_\phi(\chi', \theta)) =  \Lrn_\phi( \chi\cseq\chi',\theta)$%
        \label{ax:combinativity}%
\end{LrnAxioms}
\vnew{%
\cref{ax:combinativity} appears to be a rather strong assumption. 
Since $\top$ is absorbing,
for example, \cref{ax:combinativity}
implies
\cref{ax:idemp}.
In the language of algebra, \cref{ax:zero} and \ref{ax:combinativity} (and \cref{ax:cont-and-smooth}) together require $\Lrn_\phi$ to be a (smooth) {action} of the monoid $(\confdom, \cseq, \bot)$ on $\Theta$.
}%
However, if we are free to chose the confidence domain, \cref{ax:combinativity} imposes no other restrictions on $\Lrn$ 
	(see \cref{prop:free-additivity} in the appendix). 
It is also easy to verify that the confidences $\alpha$ and $n$ of \cref{ex:prob-simple,ex:shafer,ex:classifier} satisfy \cref{ax:combinativity}.
Nevertheless, for the canonical domains $[0,1]$ and $[0,\infty]$, 
\cref{ax:combinativity} is indeed a strong assumption.
In fact, of the confidences in \cref{ex:kalman1d}, neither $K$ alone nor $\sigma^2$ satisfy \cref{ax:combinativity} out of the box---but sensor precision $\sigma^{-2}$ does when $K = K_{\text{opt}}$ is the optimal gain, 
	 and the pair $(K,\sigma^2)$ can be combined into a single domain satisfying \cref{ax:combinativity}, as we show in the appendix.
\commentout{\color{gray}%
Keep in mind that \cref{ax:combinativity} applies only for two
	observations of the same statement $\phi$.
}%

\commentout{%
This suggests that we could use $\Lrn$ to define the 
 	belief states in which ``$\phi$ is true'' to be
	image of this projection (i.e., the set of fixed points of $\smash{\Lrn^{\top}_\phi}$);
after all, it is easily shown that learnining $\phi$ with any degree of confidence (i.e., applying $\Lrn_\phi^\chi$) has no effect on these states. 
This illustrates a general point: if the function $\Lrn$ captures the
	belief updating process, we can use it to understand the relationship between $\Phi$ and $\Theta$ at an abstract level.
In \cref{ex:classifier}, for instance, 
		although the network weights $\Theta$ are an uninterpreted subset of some high dimensional space, 
		the training process $\Lrn$ 
		arguably imbues them with meaning by defining
			a connection between them and the training examples.
					
However, to some readers, using $\Lrn$ to define truth may seem backwards.
In a given learning setting, we may already have a sense of
which belief states $\theta$ correspond to full belief in $\phi$---in \cref{ex:prob-simple}, for instance, 
	they are the measures that give $\phi$ probability 1.
In such cases, we may want additional axioms ensuring that 
	any relationships between $\Theta$ and $\Phi$ implicit in $\Lrn$
	are compatible with the ones we already have. 
}
Our axioms so far have been conditions on the separate commitment functions $F: \confdom \times \Theta \to \Theta$, which we have called ``$\Lrn_\phi$'',
	but we have not required that $F= \Lrn_\phi$ have any relationship to observations $\phi$. 
To address this, we must reintroduce the final pieces of our formalism.

\commentout{%
Beyond the data of this monoid, a confidence domain $D$ also has 
	an order ($\le$),
	a geometry,
	and an absorbing top element ($\top$). 
What should it mean for $\Lrn$ to preserve this additional structure? 
}%

\subsection{Observations and Degree of Belief}
	\label{ssec:full-learn}

In a learning setting $(\Theta, \Phi,\confdom)$, suppose that we have a function $\Bel : \Theta \times \Phi \to \confdom$
that associates to a belief state $\theta$ 
	a degree of belief in each statement $\phi$. 
The output of $\Bel$ is an ``epistemic'' confidence rather than a learner's confidence (recall the difference between senses \ref{item:epistemic-conf} and \ref{item:learn-conf} at the end of the introduction).
Nevertheless, such a function $\Bel$ helps capture important intuitions about the role of confidence in the learning process.
To begin, learning $\phi$ with more confidence should lead to more belief in $\phi$.

\begin{LrnBelAxioms}[nosep]
	\item 
	$\forall \phi,\theta,\chi,\chi'.\quad$
	$\chi \ge \chi'
	$\\$
	\implies
	\Bel(\phi, \Lrn(\phi,\chi,\theta)) \ge \Bel(\phi, \Lrn(\phi, \chi', \theta))
	$.
		\label{ax:monotone}
\end{LrnBelAxioms}

We cannot ask for strict monotonicity, however:
if we already fully believe $\phi$ (i.e., $\Bel(\phi,\theta) = \top$),
there is no way to attain a higher degree of belief, we cannot attain a higher degree of belief by learning $\phi$.
Instead, if we fully believe $\phi$, learning $\phi$ should
	have no effect.

\begin{LrnBelAxioms}[nosep]
\item If $\Bel(\phi,\theta) = \top$, then
    $\Lrn(\phi,\chi,\theta) = \theta$. 
    \label{ax:truth-is-enough}
\end{LrnBelAxioms}
Perhaps even more importantly, 
if we learn something with full confidence, then we ought to fully believe it.

\begin{LrnBelAxioms}[nosep]
    \item $\Bel(\phi, \Lrn(\phi,\top,\theta)) = \top$.
        \label{ax:effectiveness}
\end{LrnBelAxioms}

\commentout{%
While
\cref{ax:effectiveness} is certainly desirable,
		it may not always hold in cases of interest.
	In \cref{ex:classifier}, for instance,
		it is natural to set $\Bel(\theta,(x,y)) = f_\theta(y|x)$,
	and there may be a local maximum $\theta$ of the parameterization
		$\theta \mapsto \Bel(\theta,(x,y))$ that is not a global one.
	In this case, there is no continuous monotonic path from $\theta$ to a global maximum $\theta^*$ for which $f_{\theta^*}(y|x) = 1$,
	(i.e., no way to satisfy \cref{ax:monotone,ax:effectiveness,ax:cont-and-smooth}).
}%

While
\cref{ax:monotone,ax:effectiveness,ax:cont-and-smooth}
are serious constraints on $\Lrn$ if $\Bel$ is given,
one can easily define $\Bel$ based on $\Lrn$ so as to ensure that 
\cref{ax:monotone,ax:effectiveness,ax:cont-and-smooth}
hold trivially.
Later on (in \cref{sec:loss-repr}), we will consider an
	axiom (\cref{ax:lb-ascent})
	relating $\Lrn$ and $\Bel$
	that says far more about $\Lrn$ without specifying $\Bel$.

\commentout{%
\textbf{Symmetry.}
We would also like update rules to preserve any joint symmetries between the belief space $\Theta$ and the observation language $\Phi$.
For instance, in \cref{ex:prob-simple}, we would like to require that updates are not sensitive to irrelevant relabelings of points.
Concretely, assume we have some set $\mathrm{Aut}(\Theta, \Phi)$ of 
	structural symmetries
	(in the form of automorphisms $\sigma : (\Theta \sqcup \Phi) \to (\Theta \sqcup \Phi)$)
	that have an action both on belief states ($\sigma(\theta) \in \Theta$) and 
		on observations ($\sigma(\phi) \in \Phi$).
		The symmetry condition can now be captured by:
\begin{LrnAxioms}
	\item
	$\forall \theta,\phi,\chi,~~
	 \sigma
	\in \mathrm{Aut}(\Theta, \Phi)
	.\quad$
$\Lrn(\sigma(\phi), \chi , \sigma(\theta)) = \sigma (\Lrn(\theta,\chi, \phi))$.
	 \label{ax:symmetry}
\end{LrnAxioms}
}

\section{The Confidence Continuum}
	\label{sec:conf-continuum}
We now look deeper into the theory of learners whose confidence domain is a continuum 
(i.e., a connected, totally ordered, one-dimensional manifold with two endpoints). 
\commentout{
Most quantities used in science and everyday life can be measured additively:
if one starts with seven minutes/meters/votes/dollars,
and then gains six
more,
one has thirteen altogether.
}%

With the domain $[0,\infty]$, \cref{ax:combinativity} means $\Lrn$ is \emph{additive},
making it amenable to analogies of weight (e.g., the weight of evidence $w$ in \cref{ex:shafer})
and time (e.g., the number of training iterations $n$ in \cref{ex:classifier}).
Indeed, an additive learner can be implemented so that confidence really does coincide with time: imagine a machine with state space $\Theta$, controlled by buttons labeled by $\Phi$, that, while $\phi$ is pressed, evolves from initial state $\theta_0$ according to $\theta(t) = \Lrn(\phi, t, \theta)$. 
Conversely, this interpretation is coherent only if $\Lrn$ is additive---for otherwise there would exist $t_1,t_2$ such that the machine's state after pressing $\phi$ for $t_1$ seconds followed by $t_2$ additional seconds, would be different from the configuration after holding down $\phi$ for $t_1+t_2$ seconds.

Temporal analogies may not always be appropriate
(as they may clash with other, truer conceptions of ``time''),
yet they have such intuitive force that 
a function
$f: [a,b] \times \Theta \to \Theta$  
	(with $0 \in [a,b] \subseteq \mathbb R$)
satisfying \cref{ax:zero,ax:cont-and-smooth,ax:combinativity}
is known generically as a \emph{flow} \parencite{lee2013smooth}.
\commentout{
    Beyond \cref{ax:diffble,ax:additivity},
    $F$ need only handle full-confidence
    appropriately (i.e., satisfy \cref{ax:idemp,ax:cont})
    in order to satisfy all of our axioms thus far.}%
\commentout{%
	\begin{linked}{prop}{continuum-seqacyc}
		When $\confdom$ is a continuum, \cref{ax:zero,ax:cont-and-smooth,ax:combinativity} imply \cref{ax:seq-for-more} and \cref{ax:acyclic}.
	\end{linked}
	}%
Since \cref{ax:combinativity} implies \crefrange{ax:seq-for-more}{ax:acyclic} for this domain, 
the only additional requirement of a commitment function is that $\Lrn_{(\theta,\phi)}(\chi)$ have a well-defined limit as $\chi \to \infty$. 
This highlights the strength of the assumption that confidence lies in $[0,\infty]$ and combines additively, so one might understandably worry that this could limit applicability---but this is not the case.
While 
the additive domain $([0,\infty], +)$ certainly restricts
how confidence can be measured, it has little effect on what confidence can express.

\commentout{%
\begin{linked}{theorem}{add-reparam}
	If $\Lrn 
	$ 
	satisfies \cref{ax:zero,ax:cont-and-smooth,ax:seq-for-more,??}
	(but possibly not \cref{ax:combinativity})
	then there exists
	a flow update function
	$^+\!\Lrn$
	for the additive confidence domain $[0,\infty]$
	that does, 
	and a continuous function
	$g : \Phi \times [\bot,\!\top] \times \Theta \to [0,\infty]$
	such that
	\[
		\forall \theta,\phi,\chi.\qquad
		\Lrn( \phi,
			\chi,
		 \theta )
		 =
		{^+}\!\Lrn(\phi,~
		g(\phi,\chi,\theta),~
		 \theta)
		 \qquad\text{and}\qquad
		{^+}\!\Bel(\phi,\theta) = g(\phi, \Bel(\phi,\theta),\theta)
		.
	\]
	Furthermore,
	 $(^+\!F, g)$ is unique up to a multiplicative factor
	in the output of $g$.
	\end{linked}
\begin{coro}
There is a unique pair $(^+\!\Lrn, \beta)$
such that $^+\!\Lrn$ and $\Lrn$ have the same effect on observations
made with sufficiently low confidence, 
i.e., $\frac{\partial \beta}{\partial \chi}\big|_{\chi=\bot} = 1$.
\end{coro}
}%

\begin{linked}{theorem}{add-reparam}
	If $\confdom$ is a 
	continuum and $F : \confdom \times \Theta \to \Theta$ is a commitment function 
	(i.e., satisfies \cref{ax:zero,ax:cont-and-smooth,ax:seq-for-more,ax:acyclic,ax:combinativity}
	\unskip),
	then there exists a continuous ``translation'' function $g : \confdom \times \Theta \to [0,\infty]$,
	and a commitment flow $^+\!F$ such that
	$
		\forall \theta, \chi 
		.~
		^+\!F(g(\chi,\theta), \theta) = F(\chi, \theta). 
	$
\end{linked}

Thus, updates performed with $\Lrn$ are equivalent
to updates performed with ${^+}\!\Lrn$ (its \emph{additive form}), 
	if confidences are translated (via $g$) appropriately.
\commentout{%
\begin{defn}
We call an update function $F$ \emph{uniform} if 
the additive form
$g(\phi,\chi,\theta)$
of its confidence depends only on $\chi$
(and not on $\theta$ or $\phi$). 
\end{defn}

\cref{ax:additivity} implies uniformity, as then $^+\!F = F$ 
and
$g(\phi,\chi,\theta) = \chi$.
}%
When the original domain $\confdom$ is isomorphic to the canonical domains $[0,\infty]$ and $[0,1]$,  the translation $g$ need not depend on $\theta$ and there is a unique such representation, up to a multiplicative constant in the output of $g$. 
However, by allowing for a belief-state-dependent confidence translation, our construction provides in principle an additive representation even for very different confidence domains, such as when $\cseq$ is not invertible (e.g., $\cseq = \max$)---provided we can handle any points of non-differentiability.
This is often (but not necessarily) possible.

\vnew{%
\begin{example}
Consider a learner for the confidence domain $([0,1],\max)$, whose belief state $\theta \in [0,1]^\Phi$ is a function assigning a numerical degree of belief in each proposition, that, upon learning $\phi$, updates its posterior belief in $\phi$ according to 
$	\Lrn(\phi, \chi, \theta)(\phi) = \max \{ \chi, \theta(\phi) \}
$
while maintaining the prior belief in all other $\phi' \ne \phi$. 
Applying the construction behind the proof of \cref{theorem:add-reparam} yields the additive form
\begin{align*}
	{^+\!\Lrn}(\phi, t, \theta) &= \theta(\phi) + (1-\theta(\phi))(1- e^{-t}) \\
\intertext{with confidence translation}
	g(\phi, \chi, \theta) &= \begin{cases}
		0 &\text{ if } \chi \le \theta(\phi) \\
		\log \Big(\frac{1-\theta(\phi)}{1-\chi}\Big) &\text{ if } \chi \ge \theta(\phi).
	\end{cases}
	\qedhere
\end{align*}
\end{example}
}%

The key to proving \cref{theorem:add-reparam} is realizing that commitment flows can be equivalently represented by vector fields. This view, which we now unpack, confers other benefits as well. 

\subsection{Orderless Combination and the Vector Field Representation}
\label{sec:vecrep}

Is it the same to learn $\phi_1$ and then $\phi_2$ as it is to learn them in the opposite order? 
Order of learning does not matter for belief functions 
(\cref{ex:shafer}) or when conditioning (provided one never learns contradictory events). 
But, in general, 
the order of observations can have a significant impact on the result.
Humans tend to have a recency bias: more recent observations have a stronger influence on beliefs.
\Cref{ex:prob-simple,ex:kalman1d} are not commutative either.
But if the order matters for our update, what should we do if we receive two pieces of information simultaneously?
There  is a natural way to do this
	with the techniques used to prove \cref{theorem:add-reparam}.

\commentout{
We now turn to an equivalent representation of 
flow
update functions, which, among
other things, will ultimately
yield a natural way of
orderlessly learning $\phi_1$ and $\phi_2$ together, and weighted by relative
confidence. 
At a technical level, we show how to
extend an arbitrary update function $F$, that handles inputs $\Phi$,
to handle a more expressive set of inputs $\ext\Phi \supseteq \Phi$
closed under new operations of
orderless combination ($\cseq$), and rescaling by relative confidence ($\cdot$).
}%

Since $\Theta$ carries a differentiable structure,
it makes sense to talk about its tangent space
$T\Theta$,
which consists of pairs $(\theta, \mat v)$ where
$\theta \in \Theta$, and $\mat v$,
intuitively, is a direction that one can travel in $\Theta$ beginning at $\theta$
\parencite[\S3]{lee2013smooth}.
A \emph{vector field} $X \in \mathfrak X\Theta$ is a
differentiable
map $X : \Theta \to T \Theta$
assigning to each  $\theta \in \Theta$
a vector $X(\theta) = (\theta, \mat v) \in T\Theta$
tangent to $\theta$.
\commentout{
	The set of all vector fields over $\Theta$ is denoted $\mathfrak X(\Theta)$
	 and forms a vector space.
	\parencite[\S8]{lee2013smooth}
	\unskip.
	}%
\Cref{ax:seq-for-more}
implies that the behavior of $\Lrn$ is generated by the
way it handles updates of small confidence.
So, in a sense, all we need to know about
$\Lrn$
is how it handles infinitessimal confidences
\unskip---which can be
viewed as a vector field.
More precisely, in most cases (such as when using either the fractional or additive confidence domains),
a commitment function
$\Lrn_\phi$
can be represented by
the vector field
\begin{equation}
	\Lrn'_\phi
	:=
	\theta \mapsto
	\frac{\partial}{\partial \chi} \Lrn(\theta, \chi, \phi) \Big|_{\chi=\bot}
	\qquad\in 
	\mathfrak X \Theta
	.
	\label{eq:f-field}
\end{equation}
(To handle edge cases involving the zero field, we may need a more complex but closely related definition; see the proof of \cref{theorem:add-reparam} for details.)
We can then recover $^+\!\Lrn_\phi$ as the integral curves of $\Lrn'_\phi$
\parencite[Thm 9.12]{lee2013smooth}.
\commentout{%
\begin{fact}[{\citeauthor[Thm 9.12]{lee2013smooth}}]
	If $X \in \mathfrak X(\Theta)$,
	there is at most one
	 function
	$f : [0,\infty) \times \Theta \to \Theta$
	satisfying
	$
		f(a, f(b, \theta)) = f(a+b,\theta)
			~~\text{and}~~
		\frac{\partial}{\partial t}
			 f(t,\theta)
			|_{t{=}0}
			\!\!= X(\theta)
	$
	for all $\theta \in \Theta$ and $a,b\ge 0$.
	\label{fact:unique-integral-curves}
\end{fact}
\begin{coro}
	If $\Lrn_{\phi_1}$ and $\Lrn_{\phi_2}: \Theta \times [0,1] \to \Theta$ are distinct,
	then so are $\Lrn'_{\phi_1}$ and $\Lrn'_{\phi_2}
	$.
	\label{fact:unique-flow-for-vfield}
\end{coro}
}%
It may seem counter-intuitive that the vector field $\Lrn'_\phi$,
which does not mention confidence at all, alegedly captures confidence---%
\unskip but it does, intuitively, by specifying
everything about the learning process {except} for the degree of confidence itself.
\commentout{
This vector field representation is useful for two reasons:
at a practical level, it gives us a natural extension of $\Phi$
that allows us deal with ``mixtures'' of observations and commonly arise.
At a deeper level, it will enable us to describe and classify
the flow update functions on $\Theta$.
}%

The fact that it makes sense to add vector fields (and the result does not depend on the order of addition) suggests a way of handling simultaneous parallel observations. 
\commentout{%
We now return to orderless combination of observations.
One key property of vector fields is thier closure under linear combination---and since
commitment flows and vector fields are equivalent, 
we can extend this linear structure to observations themselves.
This gives us a natural way to combine observations in parallel.
}%
\commentout{%
There are two aspects of linearity: scalar multiplication, and addition.
From scalar mutliplication, we get a way of rescaling
inputs
by a ``relative confidence'' $k
$.
Concretely, given $\phi \in \Phi$ and $k \in (0,\infty)$,
define a new observation
$k\cdot\phi$
and extend $F$ to
a function $\ext \Lrn$ that handles it by:
$
	\ext \Lrn^{\chi}_{k\cdot\phi}(\theta) := \Lrn^{k\chi}_{\phi}(\theta)
$
\unskip, or equivalently, $
	\Lrn'_{k\cdot \phi} := k \Lrn'_{\phi}
	.
$
The rescaled input
$k\cdot \phi$ behaves the same way that $\phi$
does for extreme values of confidence,
since $k 0 = 0$ and $k\infty = \infty$.
This is precisely the same degree of freedom as exposed in  \cref{prop:az-iso}.
}%
Given $\phi_1, \phi_2 \in \Phi$, we can form a new observation
$\phi_1 \oplus \phi_2$ and extend $\Lrn$ to handle it implicitly, via its vector field representation, according to
$\Lrn'_{\phi_1 \oplus \phi_2} := \Lrn'_{\phi_1} + \Lrn'_{\phi_2}$.
Standard existence theorems and uniqueness theorems for ordinary differential equations then apply, delivering a commitment function in additive form, modulo the following two caveats:
(1)
$\lim_{t \to \infty} \Lrn^{t}_{\phi_1 \oplus \phi_2}$ may not exist in some cases,
in which case we cannot continuously extend $\Lrn_{\phi_1 \oplus \phi_2}$ to handle full confidence, and
(2) $\Lrn_{\phi_1\oplus\phi_2}$ might not satisfy \cref{ax:acyclic}.
We leave $\phi_1 \oplus \phi_2$ undefined in such cases, 
but point out that satisfying axiom \cref{ax:lb-ascent} (the subject of \cref{sec:loss-repr})
suffices to prevent both problems. 
\commentout{%
\begin{prop}
	If $\Lrn$ is a flow update function
	then the following are equivalent:
	\begin{enumerate}
		\item $\Lrn_{\phi_1}^{\chi_1} \circ \Lrn_{\phi_2}^{\chi_2} =  \Lrn_{\phi_2}^{\chi_2} \circ \Lrn_{\phi_1}^{\chi_1}$
		for some $\chi_1, \chi_2 \notin \{\bot,\top\}$.
		\item $\Lrn_{\phi_1}^{\chi_1} \circ \Lrn_{\phi_2}^{\chi_2} =  \Lrn_{\phi_2}^{\chi_2} \circ \Lrn_{\phi_1}^{\chi_1}$
		for all $\chi_1, \chi_2 \notin \{\bot,\top\}$.

		\item The vector fields $\Lrn'_{\phi_1}$ and $\Lrn'_{\phi_2}$ commute.

		\item
			For all $\chi \in \mathbb R$, 
			$\Lrn^{\chi}_{\phi_1} \circ \Lrn^{\chi}_{\phi_2} = \Lrn^\chi_{\phi_1\oplus\phi_2}$.
	\end{enumerate}
	If this condition holds, then $\phi_1$ and $\phi_2$ are said to \emph{commute}.
\end{prop}
}%

\commentout{
	Observations $\phi_1$ and $\phi_2$ \emph{commute} at $\theta$ iff $\Lrn_{\phi_1}^{\chi_1} \circ \Lrn_{\phi_2}^{\chi_2}(\theta) =  \Lrn_{\phi_2}^{\chi_2} \circ \Lrn_{\phi_1}^{\chi_1}(\theta)$
for all $\chi_1, \chi_2 \neq \top$. 
}
\commentout{%
\begin{prop}
	If $\Lrn^{\chi}_{\phi_1} \circ \Lrn^{\chi}_{\phi_2} =
		\Lrn^{\chi}_{\phi_2} \circ \Lrn^{\chi}_{\phi_1}$,
	then both updates are equal to
	 $\Lrn^{\chi}_{\phi_1 \oplus \phi_2}$. %
	\commentout{That is,
	\[
		F^{\chi}_{\phi_1}( F^{\chi}_{\phi_2}(\theta))
		=
		F^{\chi}_{\phi_2 \oplus \phi_1} (\theta)
		=
		F^{\chi}_{\phi_1 \oplus \phi_2} (\theta)
		=
		F^{\chi}_{\phi_1}( F^{\chi}_{\phi_2}(\theta))
		.
	\]}
\end{prop}
}%

We now illustrate what we've just done by example. 
\begin{example}
Recall the learner of \cref{ex:classifier}, a classifier that updates with gradient descent. 
In this case, simultaneous parallel observations correspond to estimating gradients using a mini-batch of training samples, rather than one at a time. This standard practice is known to stabilize training. Indeed, use of the full gradient across all training examples amounts to their simultaneous observation in this sense. 
\end{example}
\begin{example}
	Recall the learner of \cref{ex:prob-simple} that linearly interpolates towards conditioning. 
	Learning $A$ and $B$ in sequence with high confidence amounts to conditioning on their intersection, and is undefined when $A \cap B = \emptyset$. 
	At lower confidence values, however, the order of the observations matters 
	(i.e., $A$ and $B$ do not commute) 
	except at prior belief states $P$ according to which $A$ and $B$ are independent, i.e., $P(A \cap B) = P(A)P(B)$. 
	Yet our construction gives us a natural way of simultaneously observing events that are not independent---even ones that are contradictory.
	For example, we can learn $A \oplus \lnot A$.
	At high confidence, this leads to the posterior $\frac12 (P|A) + \frac12 (P|\lnot A)$, reflecting maximal uncertainty about the truth of $A$. 
\end{example}

Clearly $\phi_1 \oplus \phi_2 = \phi_2 \oplus \phi_1$ when either is
defined, so $\oplus$ provides a way of combining observations
orderlessly, even in cases where $\phi_1$ and $\phi_2$ do not commute%
---and when they do, $\phi_1\oplus \phi_2$
is equivalent to observing $\phi_1$ and $\phi_2$ in either order.
This follows from following proposition, which demonstrates that $\phi_1\oplus\phi_2$ can be thought of as an infinitely
fine interleaving of low-confidence $\phi_1$ and $\phi_2$ updates.

\begin{linked}{prop}{linterleave}
	Suppose $\Lrn_{\phi_1}$ and $\Lrn_{\phi_2}$ are commitment flows.
	For $t \in [0, \infty]$, 
	let
	$L_t := \Lrn_{\phi_2}^t \circ \Lrn_{\phi_1}^t
	$
 	denote
	learning $\phi_1$ followed by $\phi_2$ (both with confidence $t$), 
	and
	for $n \in \mathbb N$, let
	$L_t^{(n)}(\theta) := L_t \circ\cdots\circ L_t(\theta)$
	denote $n$ repeated applications of $L_t$.
	Then
	$
		\Lrn_{\phi_1 \oplus \phi_2}^\chi(\theta) =
			\lim\limits_{n \to \infty} L_{\nicefrac\chi n}^{(n)}(\theta)
		.
	$
	\onlyfirsttime{\unskip\footnote{
	For completeness, note that \cref{prop:linterleave} is closely related to the \emph{Lie-Trotter product formula} \citep{trotter1959product,cohen1982eigenvalue}, and can be viewed as an interpreted instantiation of it.
	}}
\end{linked}

\commentout{%
\paragraph{Vector Field Representations and Control Theory.}
In many ways, 
	the assumptions we have made in \cref{sec:vecrep}
	have lead our framework to resemble a dynamical system. 
We have a continuous manifold of states $\Theta$, a set of 
	of inputs (``control signals'') $\Phi$, which cause $\Theta$ 
	to evolve ``over time''. 
However, there are two critical differences. 
First, control theory does not require the analogue of a ``full-confidence'' update; there may be no limit as $t \to \infty$.
Thus, while control theory must apply to arbitrary dynamical systems, 
	the theory of confidence needs only describe those which describe motion that uniformly approaches a fixed point. %
Second, while ``time'' has a single clear interpretation in control theory, our analogue of additive confidence is only well-defined up to a multiplicative constant.
In some cases, the analogy can break down significantly---%
	observing $\phi_2$ after learning $\phi_1$ with full confidence, for instance, 
	extends ``time'' past $t=\infty$.
Finally, we mention that in some contexts, it is clearer to think of confidence in the range $[0,1]$, never adopting a temporal analogy at all. 
}%
\commentout{
Going back through our examples:
\begin{description}
	\item[{\bf[\cref{ex:prob-simple}]}]
		$g(\mu, \alpha, \phi) = - \log(1-\alpha)$.
		This means that
		$^+\!F(\mu, \beta, \phi) = e^{-\beta} \mu + (1-e^{-\beta}) (\mu\mid \phi)$.
		
	\item [{\bf[\cref{ex:shafer}]}] 
		Weight of evidence $w$ is already additive, so
			$g(m, w, \phi) = w$, and $^+\!F = F$. 
		Meanwhile, degree of support $\alpha$ is translated
		the same way as $\alpha$ in the first example: in this case,
			$g(m, \alpha, \phi) = - \log(1-\alpha)$. 
		
		\commentout{
		As noted in the introduction,
		the restriction of this update rule to belief states
		that are probabilities, gives an update rule}

	\item [{\bf[\cref{ex:classifier}]}] 
		($n$ is already additive).
\end{description}
}%

\subsection{Optimizing Learners}
\label{sec:loss-repr}

\def\GD#1{\mathtt{GradFlow}[#1]}
\def\NGD#1{\mathtt{NGF}[#1]}

We have now seen how learners satisfying certain axioms
can be represented as vector fields (\cref{sec:vecrep}).
A particularly important way of specifying a vector field is via the gradient of a potential.
This is especially true in modern machine learning, where training is idealized as loss-minimizing gradient flow \citep{arora2018convergence}, and where the substantial
advances of the last two decades have repeatedly demonstrated value of casting learning as optimization \citep{sra2011optimization}.
Our framework allows us to express this idea as a simple relationship between $\Lrn$ and $\Bel$:

\begin{LrnBelAxioms}[nosep]
	\item $\displaystyle
	\frac{\partial}{\partial \chi} 
		{\Lrn}
		(\phi, \chi, \theta)
	= \nabla_{\mskip-2mu\theta }
		\,
		{\Bel}
			(\phi, \theta)
	$ \label{ax:lb-ascent}
\end{LrnBelAxioms}

\Cref{ax:lb-ascent} says that learning occurs by gradient ascent (%
i.e., using some measure of disbelief in observations
as a loss):
that learning is fundamentally (just) about locally increasing degree of belief---no more, and no less.
It also gives us a way of turning $\Bel$ (whose output is an epistemic confidence) into a commitment flow $\Lrn$ (which takes a learner's confidence as input), which may have contributed to any ambient confusion about the distinction between the two readings of the word ``confidence''. 
Unlike \cref{ax:effectiveness,ax:monotone,ax:truth-is-enough}, 
	\cref{ax:lb-ascent} imposes serious constraints on $\Lrn$ even if we are free to select $\Bel$.
We say $\Lrn$ is \emph{optimizing} if there exists some $\Bel$ such that the pair satisfy \cref{ax:lb-ascent}. 
This way of constructing a learner has another benefit: the flows formed from such vector fields are guaranteed to have limits and satisfy \cref{ax:acyclic}, meaning that orderless combination $\oplus$ is always well-defined.

\commentout{
To do this at full generality, we need to make sense of a gradient, which requires more structure, in the form of a Riemannian metric.  It turns out that, up to a multiplicative constant, there is a unique natural Riemannian metric on any parameterization of a probability distributions \cite{chentsov}; taking gradients with respect to this geometry, show how familiar loss functions on probability measures correspond to different standard notions of confidence in the other representations.

Suppose we have:
\begin{enumerate}[nosep]
	\item A differentiable loss function $\mathcal L : \Theta \times \Phi  \to \mathbb R$, which intuitively measures the ``incompatibility'' between a belief state $\theta$ and an assertion $\varphi$, and
	\item
		A way of taking the gradient of ${\cal L}$ with respect to $\theta$,%
			\footnote{
			such as a tangent-cotangent isomorphism $(-)^\sharp : T^*_p\Theta \to T_p \Theta$, perhaps coming from an affine connection, in turn perhaps coming from a Riemmannian metric.}
        so as to obtain a vector field on $\Theta$ which optimizes $\mathcal L.$
\end{enumerate}

Then we can define an update rule $\GD {\cal L}$ that reduces inconsistency by gradient flow (the continuous limit of gradient descent). Concretely, such an update rule has a vector field:
\[
	\GD {\cal L}'_\phi(\theta) = - \nabla_\theta {\cal L}(\theta,\phi).
\]

\begin{prop}
	An update rule $F$ on a Riemannian manifold $\Theta$ is optimizing update rule if and only if $(F')^\flat$ is a conservative co-vector field.
	\cite[Prop 11.40]{lee2013smooth}
\end{prop}
}

Technically, to view the derivative of a function $\ell : \Theta \to \mathbb R$ as a vector field $\nabla \ell \in \mathfrak X\Theta$ (rather than a co-vector field), one needs more than a manifold structure on $\Theta$; we will assume that $\Theta$ comes with what is called a \emph{Riemannian Metric}. 
The details are unimportant; 
what matters is that we can always fall back on the Euclidean metric for subsets of $\mathbb R^n$, and that some other spaces (such as parametric families of distributions), have a different natural metric.

\textbf{Optimizing Commitment for Probabilistic Beliefs.}
In many learning settings of interest,
beliefs $\theta \in \Theta$ are associated
with probability distributions $P_\theta \in \Delta \Omega$ over some measurable space $\Omega$. 
Fortunately, this gives us a natural Riemannian metric on $\Theta$---which, as explained above, is precisely what we need in order to make sense of gradients on a manifold. 
Specifically, the \emph{Fisher Information Metric} (FIM) induced by the parameterization $\theta \mapsto P_\theta$ 
turns out to be the unique metric (up to scalar multiple) that 
	is invariant under sufficient statistics
	\citep{chentsov}%
\footnote{
For instance, 
if $X$ and $Y$ take values in $\Omega$, and 
$p(Y|X)$ and $q(X|Y)$ are such that
$P_{\theta}(X) = q \circ p \circ P_\theta(X)$ 
for all $\theta$,
\commentout{
as depicted by the in commutative diagram
	\[
	\begin{tikzcd}[ampersand replacement =\&]
		\Theta \ar[r,"\Pr"]\ar[d,"\Pr"']
			\& X \\
		X \ar[r,"p"'] \& Y \ar[u, "q"']
	\end{tikzcd}, 
	\]
}%
then clearly the family $P_\theta(Y) := p\,\circ\,P_{\theta}(X)$ carries the same information about the parameters (and how to update them) as does $P_\theta(X)$.
Chentsov's theorem (\citeyear{chentsov}) tells us that the FIM 
	is the only Riemannian metric on $\Theta$ (as a function of the parameterization $\theta \mapsto P_\theta$), 
	that is the same whether derived from $P_\theta(X)$ or $P_\theta(Y)$. 
}%
---a finding that has lead many to use the term \emph{natural gradient} for gradients in this geometry, and formed the basis of Information Geometry
\citep{amari1998natural,amari2000methods}.

\setlength{\footnotesep}{2.5ex}
Using this representation-invariant geometry (whose implementation details we relegate to a very technical footnote
\unskip\footnote{%
	A \emph{Riemannian metric} consists of an inner product $\langle\cdot,\, \cdot\rangle_\theta : T_\theta\Theta \times T_\theta\Theta \to \mathbb R$ on tangent vectors at each point $\theta \in \Theta$; 
		it can therefore be viewed as a matrix $G(\theta)$ with components $G(\theta)_{i,j} = \langle e_i, e_j \rangle_\theta$, where $\{e_i\}$ are basis vectors of the tangent space $T_\theta \Theta$.
	The gradient of a function $f : \Theta \to \mathbb R$ in this geometry is then given by
	$
		\nabla_{\!\theta} 
			f(\theta)
			= G(\theta)^{\dagger}  \frac{\partial f}{\partial \theta}^{\mathsf T}(\theta)
	$
	where $G(\theta)^{\dagger} $ denotes the psuedoinverse%
	\commentout{%
		\footnote{Much of the literature assumes that the matrix $G$ is non-singular, and hence uses the inverse instead.  Many of our examples are non-singular for uninteresting reasons, and it suffices to use the pseudo-inverse here. 
		Much more can be said of the singular case; that is the domain of Singular Learning Theory \cite{slt}.}
	}
	of the matrix $G(\theta)$
	and $\frac{\partial f}{\partial \theta} = [\frac{\partial f}{\partial \theta_1}, %
	\ldots, \frac{\partial f}{\partial \theta_n}]$ is the 
	(co)-vector of partials (i.e., the transpose of the gradient of $f$ in the Euclidean metric, which is sensitive to the choice of coordinates).
	In the special case where $\Theta = \Delta W$ is itself the set of probability distributions over a finite set $W = \{1, \ldots, n\}$
	and $P_\theta = \theta$,
	the simplex representation $\theta = P = (p_1, \ldots, p_n) \in \Theta$
	with $\sum_{i} p_i = 1$ and $p_i \ge 0$,
	the Fisher Information Matrix is given by
	$G(P) =  \mathrm{diag}(\frac{1}{p_1}, \ldots, \frac{1}{p_n})$.
}),
we now revisit the examples from \cref{sec:intro}.

\begin{itemize}[wide,itemsep=0pt,topsep=0pt]
\item The update process of \cref{ex:prob-simple} can be shown to be the optimizing for log probability $\Bel(P, \phi) = \log P(\phi)$. In other words, it is about minimizing surprisal.

\item 
In \cref{ex:shafer}, $\Lrn(\Bel, \alpha, \phi) = \Bel \oplus \Bel_{(\alpha,\phi)}$
	is not optimizing in general; assuming \cref{ax:lb-ascent} often violates the equality of mixed partial derivatives. 	
However, in special case where the belief state $\Bel = \Plaus \in \Theta$ is restricted to probability measures, $\Lrn$ is optimizing with objective $\Bel(\Bel, \phi) = \Bel(\phi)$, perhaps atoning for the clash of symbols.
This differs from \cref{ex:prob-simple} only by a strictly increasing monotone function, which is why the two update rules differ only by reparameterization.
This is also the \emph{Bayesian} objective, as we will see in \cref{sec:Bayes}.

\item The learner in \cref{ex:classifier} is, by definition, an optimizing learner for $\Bel(\theta, (x,y)) = - \ell(\theta, x,y)$ to minimize loss.

\item In \cref{ex:kalman1d},
	the field generated at $K = 0$ is the gradient of
	$\Bel((\hat x, \sigma^2), z) = \frac12 (\hat x-z)^2 + \sigma^4$.
\end{itemize}
\textbf{Expected-Value Optimizing Learners.}
Having fixed the geometry on $\Theta$, 
	there is a 1-1 correspondence between optimizing commitment flows (those that satisfy \cref{ax:lb-ascent}) 
	and loss functions $\mathcal L = -\Bel_\phi : \Theta \to \mathbb R$
	up to an additive constant. 
One class of such functions stands out as a natural starting point for our investigations:
	the linear ones 
	$P \mapsto \Ex_P[V]$, 
	that is,
	expectations of of random variables $V : W \to \mathbb R$.
When $W = \{1, \ldots, n\}$, these functions are parameterized by vectors $\phi = V \in \mathbb R^n$. 
So, what learning procedure is induced by \cref{ax:lb-ascent} for linear beliefs?

\begin{linked}{prop}{boltz-expect-fields}
	Suppose $\Theta = \Delta W$
	and $\Phi$ consists of random variables $V: W \to \mathbb R$. 
	The flow form of the optimizing learner
	that has $\mathcal L = - \Bel(P, V) =  \Ex_{P}[V]$ 
	is 
	\vspace{-1ex}
	\begin{align*}
		&\Boltz
			(P, \beta, V)
			(w) :\propto
				P(w) \exp(-\beta V(w))
		\commentout{\\
			&= A \mapsto \frac
				1{\Ex_{P} [ \exp(-\beta V)]}
				{\int \exp(-\beta V(x)) \mathbbm 1_A \mathrm d P(x) }
		}
		.
	\end{align*}	
\end{linked}

This is also known as the \emph{softmax} distribution (relative to the base measure $P$)
with logits $V$ and temperature $1/\beta$.
Intuitively, larger confidence $\beta$
	reflects increasingly certainty 
	in states $w$ that have low potential $V(w)$. 
Indeed, using $\Boltz_V$ to update a distribution $P$ 
	with high confidence ($\beta \to \infty)$ 
	conditions $P$ on the minimizer(s) of $V$.
So for this learner, confidence ``tempers'' the distribution 
and coincides with the concept of thermodynamic coldness.

\commentout{%
In this thermodynamic analogy, one becomes more certain that particles are in their lowest energy states as temperature decreases. }

\begin{linked}{prop}{Boltz-props}
	\begin{enumerate}[label=(\alph*), parsep=0pt,itemsep=0.5ex]
	\item $\Boltz$ satisfies \cref{ax:zero,ax:seq-for-more,ax:cont-and-smooth,ax:acyclic,ax:combinativity}.
	\item $\Boltz$ updates commute and are invertible iff $\beta < \infty$.
	\item $\Boltz_{U \oplus V} = \Boltz_{U+V}$\ .
	\item $\Boltz_{V_1}^{\beta_1} \circ \cdots \circ \Boltz_{V_n}^{\beta_n} (P) = \Boltz(\smash{\sum\limits_{i=1}^n \beta_i V_i}, 1, P)$.
	\end{enumerate}
\end{linked}

Observe how well-behaved these learners are: any sequence of observations in any order is equivalent to a single observation of their weighted sum.
This property may come at a significant cost, however: learning in brains and artificial neural networks exhibits a recency bias, an effect which is arguably optimal for bounded agents
	\citep{wilson2014bounded,fudenberg2014learning}, or in changing environments.

\section{Boltzmann and Bayes}
    	\label{sec:Bayes}

Many believe that ``correctly'' accounting for confidence in updating (probabilistic) beliefs
    is a matter of properly applying \emph{Bayes' Rule (BR)}. 
\vnew{%
To some, this simply means that belief updates are given by conditioning (i.e., $\Lrn(\mu,\top,\phi) = \mu |\phi$ with the trivial confidence domain $\{\bot,\top\}$), in which case BR is a helpful theorem. Others reject that learning necessarily establishes a proposition in the posterior with certainty (at least as far as one's belief state is concerned); for these people, BR describes the update itself. 
We now analyze these accounts of Bayesianism within our framework.
}%
\commentout{%
For a Bayeisian a belief state is a probability 
$P \in \Delta \mathcal H$
over hypotheses
$h \in \mathcal H$,
each of which encodes a distribution 
$P(\phi \mid h)$
over possible observations.
}%
\commentout{%
\unskip---so we also have conditional probabilities of the form $P(\phi \mid h)$. 
We now explore this important special case within our formal framework.
}%

\begin{defn}
    $\Lrn$ is \emph{Bayesian}
    iff
    \begin{enumerate}
            [label=(\alph*),itemsep=0.2ex,parsep=0pt,topsep=0pt]
        \item 
            belief states 
            correspond to distinct probability distributions over a measurable space
            $\mathcal H$ of hypotheses
            (i.e., there is an injection $\theta \mapsto P_\theta : \Theta \to \Delta \mathcal H$). 
        \item 
            there is
            a measurable space $(\mathcal X, \mathcal A)$
            in which every observation $\phi$ 
            can be viewed as event
            (i.e., $\mathcal A \supseteq \Phi$)
            \unskip;
        \item there is a conditional probability (i.e., a Markov kernel)
            $P(X \mid H) : \mathcal H \to \Delta \mathcal X$,
            associating each hypothesis $h$ with a probability measure over $\mathcal X$;
        \item 
            there exists $\star \in \confdom$ such that, 
            for all $\phi$ and $\theta$,\\
            $
            P_{\!\textstyle\Lrn_\phi^\star(\theta)}\!
            ( h ) {=} P_\theta(h)  P(\phi | h)  / 
                \sum_{h'} \!P_\theta(h') P(\phi | h')
            $
            \unskip.\!\!\!\!
            \qedhere
    \end{enumerate}
\end{defn}

Item (d) is Bayes' rule, and prescribes posterior the posterior belief ``$P(H | \phi)$''.
Note that $\phi$ is not an event in the sample space $\mathcal H$,
    but in the space $\mathcal X$; we regard it as event in
    $\mathcal X \times \mathcal H$
    for the purposes of conditioning the joint measure $P(X,H) := P(X|H)P_\theta(H)$.
To obtain a new belief state of the same type as the original (i.e., a distribution over $\mathcal H$), however, we must also marginalize out $\mathcal X$.
Thus, apart from its effect on the hypotheses, $\phi$ is forgotten after the update.

In the special case where $P(X|H)$ is deterministic (i.e., theories are \emph{complete} enough to determine observations),
the extended sample space $\mathcal H \times \mathcal X$ is not meaningfully different from $\mathcal H$, and we simply update by conditioning
    (as in \cref{ex:prob-simple} with full confidence).
When $P(X|H) > 0$, however, we now prove that Bayesian updates are precisely the expected-value optimizing learners we characterized at the end of \cref{sec:loss-repr}. 
The following may be unsurprising to readers experienced in probabilistic methods, but one direction of the equivalence is subtler than it may appear.

\begin{linked}{prop}{Boltz-Bayes}
	$\Lrn$ is a Boltzmann learner for a potential $v \ge 0$ if and only if it is Bayesian with $P(\cdot \mid \cdot) > 0$. 
\end{linked}

\cref{prop:Boltz-Bayes} 
has a significant implication: 
Bayesian learners are optimizing (i.e., satisfy \cref{ax:lb-ascent}), and correspond  to a very special kind of optimizing learning 
where degree of belief can be viewed as the expectation of a fixed random variable.
This induces significant limitations on how a given belief representation can be used---for example, high confidence updates always lead to the boundary of the probability simplex.
This rules out situations like Jeffrey's rule, for which this is not the case.
This raises some interesting questions. 
Is there a generic way to capture all learners with Bayesian updates (with a necessarily much larger belief space)? 
Alternatively, are some natural learning procedures provably incompatible with the Bayesian frame?

We point out that the use of relative entropy (KL divergence) as the target of optimization, which is non-linear, has proved far more useful in the practice (e.g., for training the classifier in \cref{ex:classifier}).
Starting from optimizing learners whose loss representation is conditional relative entropy leads \citet{mixture-langs} to an alternate natural derivation of \emph{probabilistic dependency graphs} \citep{pdg-aaai}, leading well beyond ordinary probabilistic modeling to capture inconsistency and much of machine learning.

\commentout{%
In retrospect, the theorem may seem obvious;
	translating a Bayesian update to a Boltzmann learner is as simple as taking a logarithm.
	Still, there is some subtlety going the other direction. 
This close relationship between Bayesian updates and Boltzmann reweighting seems 
to be implicitly understood in the literature, 
	but to the best of the our knowledge, has not yet been fully captured.
}%

\section{Conclusion}

Metaphorically: if certainty is black and white, then probability allows for shades of gray, and learner's confidence is about \emph{transparency}. 
The idea is an old one, having been deployed many times before in various contexts;
  this paper unifies the approaches, providing axiomatic grounding for the concept writ large (\cref{ax:zero,ax:cont-and-smooth,ax:seq-for-more,ax:acyclic,ax:combinativity};
    \cref{ax:monotone,ax:effectiveness,ax:truth-is-enough,ax:lb-ascent}).
  We have identified the critical aspects of confidence in a very general setting, and related it to probabilistic notions of confidence (e.g., via \cref{ax:lb-ascent} and \cref{prop:Boltz-Bayes}).
The resulting framework connects many seemingly different representations of confidence and learning,
for an overview of which we invite the reader to revisit \cref{fig:map}.
We contend that this framework clarifies common points of confusion in  literature (see \cref{ssec:full-learn}).

There are many examples and applications of this framework. 
An obvious continuation point---a deeper analysis of which learning functions correspond to which loss functions 
when $\Theta$ is a parametric family of distributions
\unskip---has already born fruit that we were not able to cover here. 

A key question remains open: how should we decide how much confidence to place in an observation? 
With enough modeling assumptions, there can be a clear answer---such as in \cref{ex:kalman1d}, where the optimal Kalman gain is related to the current uncertainty and the variance of the sensor. However, as illustrated by the discussion in \cref{ex:classifier}, one's willingness to be influenced by an observation may not be merely a matter of probabilistic modeling.
This makes the question rather profound; we suspect that the search for a good answer will take us far beyond the present scope.
Having laid the formal and conceptual foundations, we are eager to report back on these projects in the future.

\begin{acknowledgements} %
    I would like to sincerely thank Joe Halpern, my PhD advisor, who contributed significantly to the introduction of this paper, reading dozens of earlier drafts. 
    The work was
    supported in part by AFOSR grant FA23862114029, MURI grant W911NF-19-1-0217, ARO grant
    W911NF-22-1-0061, and NSF grant FMitF-2319186.
    It was also supported by funding from NSERC and \emph{Fond de Recherche du Qu\'{e}bec}.
\end{acknowledgements}

\bibliography{conf}

\newpage

\onecolumn

\title{Learning with Confidence\\(Supplementary Material)}
\maketitle

\appendix

\section{PROOFS OF MAIN RESULTS}

We begin with the claims of the main (i.e. numbered) results. 
For convenience, we repeat the statements of the propositions before proving them. 

\recall{prop:az-iso}
\begin{lproof}\label{proof:az-iso}
    Clearly  $\varphi_\beta$ and $\varphi_\beta^{-1}$ are continuously differentiable, and one can verify with a few steps of simple algebra that the two are inverses. In both cases, the only possible wrinkle is the at the point of high confidence, but there are no problems there either, because:
    \[
        \lim_{s \to 1} \varphi_\beta(s) = \frac1\beta \lim_{s \to 1} \log \Big(\frac{1}{1-s}\Big) = \infty
        \qquad\text{and}\qquad
        \lim_{t \to \infty} \varphi_\beta^{-1}(t) = \lim_{t \to \infty} 1- e^{-\beta t} = 1.
    \]
    
    Next, we show that $\varphi_\beta$ preserves the structure of the confidence domain. 
    We just saw that $\varphi_\beta$ and $\varphi_\beta^{-1}$ preserve the top element $\top$ of both confidence domains.  It is even more immediate that it preserves the bottom element.
    It is also easy to see that both functions preserve the order (i.e., are monotonic).  For example, $\frac{\mathrm d}{\mathrm d s} \varphi_\beta(s) = \frac1{\beta(1-s)} \ge 0$. 

    Next we show that $\varphi_\beta$ and its inverse preserve independent combination $(\cseq)$.
    For $a, b \in [0,1]$, we have
    \begin{align*}
        \varphi_\beta(a \cseq b)
            &= \varphi_\beta(a + b - ab) \\
            &= -\frac1\beta \log(1 - a - b + ab) \\
            &= -\frac1\beta \log((1 - a)(1-b)) \\
            &= -\frac1\beta \log(1 - a) - \frac1\beta\log(1-b) \\
            &= \varphi_\beta(a) + \varphi_\beta(b).
    \end{align*}
    A similar calculation shows, for all $t, u \in [0,\infty]$, that
    \begin{align*}
        \varphi_\beta^{-1}(t) \cseq \varphi_\beta^{-1}(u)
            &= 1 - e^{-\beta t} + 1 - e^{-\beta u} - (1 - e^{-\beta t})(1- e^{-\beta u}) \\
            &= 2 - e^{-\beta t} - e^{-\beta u} - 1 + e^{-\beta t} + e^{-\beta u} - e^{-\beta(u+t)} \\
            &= 1 - e^{-\beta(u+t)} \\
            &= \varphi_\beta^{-1}(u + t). 
    \end{align*}

    Finally, we must show that these are the only isomorphisms between the two confidence domains.
    For this, we refer to a standard argument that is most directly seen as the solution to Cauchy's exponential functional equation $g(x + y) = g(x) g(y)$ after the change of variables $g = 1-f$.
    
    A similar argument is provided by \citet{shannon1948mathematical} in defense of entropy, and a much more direct analogue appears in the form we need by \citet{shafer1976mathematical}, who shows directly that every continuous mappings of $[0,1]$ to $[0,\infty]$ for which multiplication becomes addition in this way, must be of the form $s\mapsto - k \log(1-s)$, for some $k > 0$. 
\end{lproof}

\recall{prop:no-continuous-condition-ext}
\begin{lproof}\label{proof:no-continuous-condition-ext}
    Fix a non-empty subset $\phi \subseteq W$ and consider a function $F : \Delta W \times [0,1] \to \Delta W$
    such that $F(\mu, 0) = \mu$ and $F(\mu,1) = \mu | \phi$ whenever $\mu(\phi) > 0$. 
    Our aim is to show that $F$ cannot be continuous. 

    Fix distribution $\mu_0 \in \Delta W$ with the property that $\mu_0(\phi) = 0$. 
    For each $\delta > 0$, consider the set
    \[
        B_\delta(\mu_0)
            := \{ \mu \in \Delta W : \mathrm{TV}( \mu,  \mu_0) < \delta \}
            = \{ (1-\delta) \mu_0 + \delta P \}_{P \in \Delta W}
    \]
    of distributions within $\delta$ total variation distance
    of $\mu_0$.
    By assumption, $F(-, 1)$ updates by conditioning on $\phi$, which means all mass not on $\phi$ is removed, and the rest is renormalized. More precisely, this means $F((1-\delta) \mu_0 + \delta P, 1) =  P$ for all $\delta \in (0,1)$, and thus
    the image of $B_\delta(\mu_0)$ under $F$ is all of $\Delta W$. 
    Therefore, for every $\epsilon \in (0, 1)$,
    there cannot be $\delta > 0$ such that $\mu \in B_\delta(\mu_0)$ 
        implies $F(\mu, 1) \in B_{\epsilon}( F(\mu_0, 1) )$. 
    Thus $F$ cannot be continuous. 
\end{lproof}

\recall{prop:maximal-continuous-theta}
\begin{lproof}\label{proof:maximal-continuous-theta}
    As noted in the main text, the observation $\phi$ is not mathematically relevant to the argument; to simplify notation, we work with the commitment function $F := \Lrn_\phi : \Theta \times [\bot, \top] \to \Theta$. 
    In this context, the belief space $\Theta$ and confidence domain $\confdom$ both implicitly have topologies. Let $\tau \subseteq 2^\Theta$ denote the topology associated with $\Theta$ 
        (i.e., the collection of all open subsets of $\Theta$). 
    Given $U \subseteq \Theta$, we use the standard notation $F|_U$ to denote the restriction of the function $F$ to domain $U \times \confdom$. 

    By assumption (L2), for each fixed $\theta \in \Theta$, the function $F_\theta : \confdom \to \Theta$ is continuous. 
    Let 
    \[
        \mathcal U :=  \Big\{ U \in \tau  ~\Big|~
            F|_U : U \times \confdom \to \Theta \text{ is continuous }
            \Big\}
    \]
    be the set of all open subsets of $\Theta$ on which the restriction of $F$ is continuous. 
    Since unions of open sets are open, 
    we know that $\Theta_\phi := \bigcup \mathcal U \subseteq \Theta$ is open.
    We now show that it is the maximal open set on which $F$ is continuous, as promised by the theorem.

    Recall that a function $f : X \to Y$ is continuous iff the preimage $f^{-1}(V) = \{ x \in X : f(x) \in V\}$ of an open set $V \subseteq Y$ is itself an open set. 
    Given $V \subseteq \Theta$, observe that
    \begin{align*}
        (\theta, \chi) \in (F|_{\Theta_\phi})^{-1}(V)
        &\iff  \quad \exists U \in \mathcal U.~ \theta \in U ~\text{ and }~ F(\theta, \chi) \in V \\
        &\iff \quad \exists U \in \mathcal U.~  (\theta, \chi) \in (F|_U)^{-1}(V)  \\
        & \iff (\theta, \chi) \in \bigcup_{U \in \mathcal U} (F|_U)^{-1}(V). 
    \end{align*}
    In other words, we have shown that $(F|_{\Theta_\phi})^{-1}(V) = \bigcup_{U\in \mathcal U} (F|_U)^{-1}(V)$.
   
    It follows that the preimage $(F|_{\Theta_\phi})^{-1}(V)$ of an open set $V \subseteq \Theta$ is a union of open sets (since each $F_U$ was assumed to be continuous), and hence itself open.
    Therefore $F|_{\Theta_\phi}$ is continuous, and since $\Theta_\phi$ contains every other open set satisfying that property, it is the maximal such open set. 
\end{lproof}

We will return to \cref{theorem:add-reparam} in \cref{sec:proof-addrep}.
Previously, the following result was in the main text, but we no longer believe it important to state formally; we give it again here for completeness, as it still supports the discussion in \cref{sec:vecrep}.

\begin{linked}{prop}{at-most-one-flow}
	If $\Lrn$ is a commitment flow and $\phi_1, \phi_2 \in \Phi$,
	then there is at most one commitment flow
	$\Lrn_{\phi_1 \oplus \phi_2}
	 	: [0, \infty] \times \Theta \to \Theta$
	such that
	$\Lrn'_{\phi_1 \oplus \phi_2} = \Lrn'_{\phi_1} + \Lrn'_{\phi_2}$.
\end{linked}
\begin{lproof} \label{proof:at-most-one-flow}
    Most of the work is done by an important result in differential geometry: 
    
    \begin{fact}[The Fundemental Theorem on Flows]
        If $X \in \mathfrak X(\Theta)$ is a somoth vector field, then
        there is a unique function
        $f : \mathcal D \to \Theta$
        where $\mathcal D \subseteq \mathbb R \times \Theta$ is maximal,
        satisfying
        $
            f(a, f(b, \theta)) = f(a+b,\theta)
        $
        whenever $(a+b, \theta) \in \mathcal D$,
        and 
        $
            \frac{\partial}{\partial t}
                 f(t,\theta)
                |_{t{=}0}
                \!\!= X(\theta)
        $
        for all $(t,\theta) \in \mathcal D$. 
        \label{fact:unique-integral-curves}
    \end{fact}

    The statement above is a gloss and selective restatement of the statement of the result as presented by \citet[][Theorem 9.12]{lee2013smooth}, which inlines the definition of a flow (Equations 9.6 and 9.7).
    A further alteration: we are interested in a minor variant in which the vector field $X$ and the function of interest are not necessarily smooth (i.e., infinitely differentiable), but rather merely twice differentiable $(C^2)$. As discussed in Appendix C of Lee and more directly treated by \citet[\S4.1]{abraham2012manifolds}, precisely the same techniques suffice to establish the analogous result without assuming smoothness.

    Applying the $C^k$ analogue of \cref{fact:unique-integral-curves} to the vector field $X = \Lrn_{\phi_1\oplus\phi_2} = \Lrn_{\phi_1} + \Lrn_{\phi_2}$, we find that
    there is a unique flow $F : \mathcal D \to \Theta$ whose derivative is $X$ and whose domain $\mathcal D \subseteq \mathbb R \times \Theta$ is maximal. 
    Thus, 
    there is at most one function satisfying \cref{ax:zero,ax:combinativity,ax:cont-and-smooth}, and hence at most one commitment flow.
    The primary missing piece is that the resulting flow may no longer be \emph{complete}---following the sum of the two fields may ``leave'' the manifold $\Theta$ in finite time, and, even if it stays within the manifold, it may exhibit cyclic behavior, violating \cref{ax:acyclic} or standing in the way of a well-defined continuous completion at the limit $t \to \infty$.
    \qedhere

\end{lproof}

\recall{prop:linterleave}
\begin{lproof}\label{proof:linterleave}
    $\Lrn_{\phi_1\oplus\phi_2}^\chi(\theta)$ is, by definition, the result of integrating a vector field from $t=0$ to $t=\chi$. 
    That integration can be thought of as taking a process of taking (infinitely) many (infinitesimal) sequential steps in the direction of that field. 
    
    In the limit as $\epsilon\to 0$, 
    \[ 
    \Lrn_{\phi_1 \oplus \phi_2}^{\epsilon}(\theta_0) = \theta_0 + \epsilon \Lrn'_{\phi_1\oplus\phi_2} = \theta_0 + \epsilon \Lrn'_{\phi_1} + \epsilon \Lrn'_{\phi_2}
    \]
    can be viewed as a small linear addition to the original position (in any choice of local coordinates). 
    Yet by the same approximation, this is also what results from as an infinitesimal update of $\Lrn_{\phi_1}$ followed by $\Lrn_{\phi_2}$, which equals $L_\epsilon(\theta)$!
    As $\epsilon \to 0$, the Euler integration method of the field $\Lrn'_{\phi_1\oplus\phi_2}$ starting at $\theta$ from $t=0$ to $t=\chi$ with step size $\epsilon$, which equals $\Lrn^\chi_{\phi_1\oplus\phi_2}(\theta)$, is actually calculating $\lim_{n\to \infty} L^{(n)}_{\chi/n}(\theta)$. Therefore the two quantities are equal.
\end{lproof}

\recall{prop:boltz-expect-fields}

\begin{lproof} \label{proof:boltz-expect-fields}
    First, we calculate the vector field given by the gradient of $\Bel(\mu, V) = \Ex_\mu[V]$ in  the natural (Fisher) geometry for $\Theta = \Delta X$.
    \begin{align*}
         \hat\nabla_\mu \Bel(\mu, V) 
        &= \hat\nabla_\mu \Ex_\mu[V] \\
        &= \mathcal I(\mu)^{-1} ( \nabla_\mu \Ex\nolimits_\mu[V] - \lambda \mathbf 1) \\
    \intertext{ where $\lambda$ is the Lagrange multiplier associated with the constraint $g(\mu) = \sum_{x} \mu(x) - 1 = 0$, which has gradient $\nabla_\mu g(\mu) = \mat 1$.  The field is therefore given by}
        &= \left[
            \mu(x)
            \frac{\partial}{\partial \mu(x) } \Ex_\mu[V] - \lambda \mu(x)
        \right]_{x \in X} \\
        &= 
        x \mapsto \mu(x) (V(x) - \lambda)
    \end{align*} 
    for some constant $\lambda$.  We can solve for $\lambda$ with the observation that the result must yield a vector tangent to the probability simplex, i.e., the sum across all components must equal zero; thus $\sum_{x \in X} \mu(x)( V(x) - \lambda) = \Ex_\mu[V] - \lambda = 0$, and so we must have $\lambda = \Ex_\mu[V]$. Therefore,
    \begin{align*}
        \hat\nabla_\mu \Bel(\mu, V) &= x \mapsto \mu(x) (V(x) - \Ex\nolimits_\mu[V]) \\
        &= \mu \odot ( V - \Ex\nolimits_\mu[V] ),
    \end{align*}
    where $\odot$ is used to emphasize that it is an element-wise product between vectors. 
    
    At the same time, we can calculate the path velocity of the Boltzman update rule.
    Letting $Z := \Ex_\mu[ \exp(- \beta V) ]$ be the normalization constant, 
    $
        \frac{\partial Z}{\partial \beta} 
        = \Ex_\mu\left[ \frac{\partial}{\partial\beta} \exp(-\beta V)\right]
        = \Ex_\mu[ -V \exp(-\beta V) ]  
    $. 
    Keeping that in mind, we can calculate: 
    \begin{align*}
        \frac{\partial}{\partial \beta} \Boltz[V](\mu, \beta) \,\Big|_{\beta=0}
        &= x \mapsto  \frac{\partial}{\partial \beta} \Big[ \frac{\mu(x) \exp(-\beta V(x))}{\Ex_\mu[\exp(-\beta V)]}\Big] \\
        &= x \mapsto \mu(x) \frac{\partial}{\partial \beta} \Big[ \exp(-\beta V(x)) \Big]_{\beta=0} + \mu(x) \exp(-\beta V(x)) \frac{\partial}{\partial \beta} \Big[ \frac1Z \Big]_{\beta=0}
        \\
        &=  x \mapsto \mu(x) \exp(-\beta V(x)) \Big( -V(x) +  \frac{\partial}{\partial \beta} \Big[ \frac1Z \Big]_{\beta=0} \Big) \Big|_{\beta=0}
        \\
        &=  x \mapsto \mu(x) \Big( -V(x) - \frac{1}{Z^2} \frac{\partial Z}{\partial \beta} \Big)
        \\
        &=  x \mapsto \mu(x) \Big( -V(x) - \frac{\Ex_\mu[ -V \exp(- \beta V) ]}{\Ex_\mu[ \exp(- \beta V) ]^2} \Big|_{\beta=0} \Big)
        \\
        &=  x \mapsto \mu(x) ( -V(x) - \Ex\nolimits_\mu[ -V])
        \\
        &= \mu \odot ( \Ex\nolimits_\mu[V] - V)
        .
    \end{align*}
    Since this is the same field as before, \cref{prop:at-most-one-flow} tells us that $\Boltz_V$ is the unique flow representation of the optimizing learner with potential $\Ex_\mu[V]$. 
\end{lproof}

\recall{prop:Boltz-props}
\begin{lproof}\label{proof:Boltz-props}
    (a) L1 and L2 are obvious. 
    L4 follows from the fact that (as shown in \cref{prop:boltz-expect-fields}), the field is the gradient of a potential, and so it cannot have closed integral curves. 
    L5 is actually part (c), and L3 follows from L5 and the fact that adding numbers makes them larger.

    (b) 
    Boltzmann updates commute because 
    \[ 
        \Boltz_u^{\beta_1} \circ \Boltz_v^{\beta_2}(\mu)
            \propto \mu \exp(-\beta_1 u) \exp(-\beta_2 v)
            =\mu \exp(-\beta_2 v)\exp(-\beta_1 u) 
            \propto \Boltz_v^{\beta_2}\circ\Boltz_u^{\beta_1}(\mu).
    \]
    If $\beta < \infty$, the update $\Boltz_u^{\beta}$ can be inverted by $\Boltz^\beta{k-u}$ where $k$ is any constant. 
    If $\beta = \infty$, then it amounts to conditioning, and hence is not invertible. 

    (c)
    Adding the vector fields discovered in the proof of \cref{prop:boltz-expect-fields},
    \begin{align*}
        \Boltz'_{u \oplus v} 
        &= \Boltz'_u  + \Boltz'_v \\
        &=  \mu \odot (\Ex\nolimits_\mu[u] - u) + \mu \odot (\Ex\nolimits_\mu[v] - v) \\
        &= \mu \odot (\Ex\nolimits_\mu[u+v] - (u+v)) \\
        &= \Boltz'_{u+v}. \qedhere
    \end{align*}

    (d) Slightly generalizing the calculation of part (b):
    \begin{align*}
        \Boltz_{v_1}^{\beta_1} \circ \cdots \circ \Boltz_{v_n}^{\beta_n} (\mu)
            &\propto \mu \prod_{i=1}^n \exp(- \beta_i v_i) \\
            &\propto \mu \exp\Big( - \sum_{i=1}^n \beta_i v_i \Big) \\
            &\propto \Boltz_{\sum\limits_{i=1}^n \beta_i v_i}
    \end{align*}
\end{lproof}

\recall{prop:Boltz-Bayes}
\begin{lproof}\label{proof:Boltz-Bayes}
    One direction is easy: if $\Lrn$ is Bayesian with likelihood $P(\,\cdot\mid\cdot\,)>0$, then belief states are probability distributions, and so for $\star := \beta = 1$, a Bayesian update with likelihood $P(X \mid H)$ can be written as
    \begin{align*}
        P_{\Lrn(\theta,\star,\phi)}(h) &\propto P_\theta(h) \cdot P(\phi \mid h) \\
            &\propto P_\theta(h) \cdot \exp( \log P(\phi \mid h)),
    \end{align*} 
    and so coincides with the Boltzmann update with confidence 1 and potential $- \log P(\phi \mid h)$. This simple well-known fact is largely responsible for the prevalence of ``tempering'' and exponential families in the Bayesian literature. In effect, it just converts between the additive and multiplicative domains. 

    The opposite direction is less well-known, and considerably less intuitive.
    We cannot simply invert the construction above, because, owing to the fact that probabilities are constrained to sum to one, not every potential can be obtained by the logarithm of a conditional probability in this way.
    However, we can circumvent this by choosing a new measurable space $\mathcal X$.

    Concretely, suppose we are given a potential $u : \Phi \times \mathcal H \to [0, \infty)$. 
    In this case, define $X$ to be a variable whose can take on values $2^\Phi$, and define the likelihood $P(X | h)$ according to:
    \[
        P(X{=}A \mid h) := \prod_{\phi \in A} \exp(-u(\phi,h)) \prod_{\phi \in \bar A} (1-\exp(-u(\phi, h))).
    \] 
    It is not hard to see that this implies
    \[
        P(X{\supseteq}A \mid h) = \prod_{\phi \in A} \exp(-u(\phi,h)) = \exp(-\sum_{\phi \in A} u(\phi, h)).
    \]
    By viewing an observation $\phi$ as the event $X \supseteq \{\phi\}$, we now have an event whose (strictly positive) likelihood corresponds to the potential $u(\phi, -)$.
    This establishes the reverse direction of the theorem.
\end{lproof}

\commentout{%
    \recall{prop:continuum-seqacyc}
    \begin{lproof} \label{proof:continuum-seqacyc}
        Assume that $\confdom$ is a continuum (i.e., totally ordered, one-dimensional, and connected).
        Furthermore, assume that $F : \confdom \times \Theta \to \Theta$ satisfies \cref{ax:zero,ax:cont-and-smooth,ax:combinativity}.
        
        To establish \cref{ax:seq-for-more}, 
        suppose $\chi_1 \le \chi_2$, and choose $\theta \in \Theta$.

        By \cref{ax:combinativity}, we must find $\chi'' \in (\bot, \chi']$ such that
        $F(\chi'' \cseq \chi, \theta) = F(\chi', \theta)$. 
        
        By \cref{ax:zero}, we know that $F(\bot, \theta) = \theta$, and by \cref{ax:cont-and-smooth}, we know that $F_\theta |_{[\bot,\chi']} : [\bot,\chi']$ is a continuous path from $\theta$ to $F(\chi',\theta)$ that passes through $F(\chi,\theta)$ (since $\chi < \chi'$, and $\confdom$ is totally ordered and connected).  
    \end{lproof}
}%

\subsection{The Additive Representation Theorem}
    \label{sec:proof-addrep}
The proof of \cref{theorem:add-reparam} is a bit more techncial than the others. 
We will first need a technical result about differential geometry.
In this section we assume that is a continuum (a one-dimensional, totally ordered confidence domain),
and that  $F : \confdom \times \Theta \to \Theta$ is a commitment function (satisfying \cref{ax:zero,ax:cont-and-smooth,ax:seq-for-more,ax:acyclic,ax:combinativity}).

Now a few definitions. A point $p = (\chi,\theta) \in \confdom \times \Theta$
is called \emph{active} if $\frac{\partial F}{\partial \chi}|_p \ne 0$. 
$p$ is a \emph{submersion point}, or \emph{submersive}, if $d F|_p : T_p(\confdom \times \Theta) \to T_{F(p)}\Theta$ is surjective. 
(That is, if $F$ is a submersion at $p$.)

\begin{lemma}\label{lem:active-sub}
    For all $\theta \in \Theta$, if there exists an active point $p$ in the fiber $F^{-1}(\theta)$, then there also exists an active point $\hat p$ in the fiber that is a submersion point. 
\end{lemma}
\begin{proof}
    For the sake of contradition, suppose otherwise---that $p^* = (\chi^*, \theta_0) \in F^{-1}(\theta)$ is an active point in the fiber $F^{-1}(\theta)$, but no submersion point in the fiber is active (i.e., $\pd F \chi |_{p} = 0$). 
    
    Select a sequence of strictly increasing confidences $(\chi_n) \in \confdom^{\mathbb N}$ that approach $\chi^*$ from below. (So $(\chi_n) \to \chi^*$.)
    For each $n$, define $\theta_n := F(\chi_n, \theta_0)$. Since $F$ is continuous, $(\theta_n) \to F(\chi^*,\theta_0) = \theta$. 
    By \cref{ax:seq-for-more}, since $\chi_n < \chi^*$, we are guaranteed that there exists some $\delta_n \le \chi^*$ such that $F(\delta_n, F(\chi_n, \theta_0)) = \theta$, which we use to define the sequence $(\delta_n)_{n \in \mathbb N}$.  
    Defining $p_n := (\delta_n, \theta_n)$ gives a sequence of points, each lying in the fiber $F^{-1}(\theta)$ owing from the definitions of $\theta_n$ and $\delta_n$. Note that $(p_n) \to (\delta_{\lim}, \theta)$. 
    Since $\confdom$ is homeomorphic to an interval, it is bounded, so by the Bolzano-Weierstrass theorem, $(\delta_n)$ has a convergent subsequence; let $(\delta_m)$ be such a subsequence limiting to the smallest possible value (i.e., $\lim_{m\to\infty} \delta_m = \lim\inf_{n \to \infty} \delta_n =: \delta_{\lim}$).

    Define also the sequence $(q_n = (\chi_n, \delta_n))_{n \in \mathbb N}$.
    Intuitively, each $q_n = (\chi_n,\delta_n)$ is a different way of splitting up the effective total confidence $\chi_n \cseq \delta_n \cong \chi^*$.

    Intuitively, as $\chi_n$ approaches $\chi^*$, the remaining residual confidence $\delta_n$ required to effectively get there should decrease to $\bot$. 
    Indeed, 
    \[
        \lim_{n\to \infty} \delta_n 
        = \lim_{n\to \infty}\delta(\chi^*,\chi_n, \theta_0) 
        = \delta(\chi^*, \lim_{n\to\infty}\chi_n, \theta_0)
        = \delta(\chi^*,\chi^*, \theta_0) = \bot. 
    \]
    
    The point $p_\bot = (\bot, \theta)$, which is obviously in the fiber $F^{-1}(\theta)$, is a submersion point---since $F(\bot, \,\cdot\,) = \mathrm{id}_\Theta$ is the identity map on $\Theta$, it follows that $\pd F\theta|_{p_\bot}$ is the identity map on $T_\theta \Theta$ (i.e., the identity matrix in any coordinate representation). 
    This is a sufficient condition for the differential of $F$ to be surjective at this point, even if the derivative with respect to $\chi$ is zero. 
    Furthermore, since the set of invertable matrices is open and $F$ is $C^1$ along the line $\{\bot\}\times\Theta$, it follows that any point sufficiently close to that line (i.e., with small enough value of $\chi$) will be a submersion point as well. 
    
    Define the function $H(\chi, \delta) := F(\delta, F(\chi,\theta_0)) : \confdom^2 \to \Theta$, whose utility we will see shortly. 
    The level set $H^{-1}(\theta)$ consists of confidence pairs $(\chi,\delta)$ for which $F(\delta \cseq \chi,\theta_0) = \theta$ for which sequential application leads to our target. 
    At the point $p_n$, what direction keeps us within this set? 
    Taking the differential of $H$ at the point $p_n$, by the chain rule, we find: 
    \begin{equation}
        dH|_{p_n} (v) 
            = v_\delta \Big( \pd F\chi{(\delta_n,\theta_n)} \Big) + v_\chi \Big( \pd F\theta {(\delta_n,\theta_n)}  \pd F\chi{(\chi_n,\theta_0)} \Big),
            \label{eq:dH}
    \end{equation}
    for a vector $v = v_\delta \pd{}\delta + v_\chi \pd{}\chi \in T_{p_n}\confdom^2$ tangent to $p_n$.
    We are looking for vectors in the kernel of $dH|_{p_n}$ (i.e., for which $dH_{p_n}(v) = 0$); these are the ones that lie tangent to the level set of interest.
    \unskip\footnote{
    In more detail: since this differential has constant rank at a neighborhood of the limiting point (as we are about to show), the points lie on a smooth sub-manifold, by the constant rank level subset theorem.
    That submanifold is a one-dimensional curve the primary argument in the proof of \cref{theorem:add-reparam}---from \cref{ax:seq-for-more,ax:cont-and-smooth,ax:combinativity}, it follows that all $\pd F \chi$.
    }
    Remarkably, this relates the conditions of activeness and submersiveness at $p_n$ to activeness at the point $p^*$, which was guaranteed by assumption!
    \begin{itemize}
        \item  By our assumption that $p^* = (\chi^*,\theta_0)$ is active, the derivative
            $\pd F \chi|_{(\chi^*,\theta_0)} =: v^*$ exists and is a nonzero tangent vector;
            moreover, that nonzero value is the limit of the sequence $(\pd F\chi(\chi_n, \theta_0) )_{n \in \mathbb N}$.
            Therefore, for $\epsilon > 0$ there exists an integer $N_1$ for which $\pd F\chi(\chi_n, \theta_0)$ is within $\epsilon$ of $v^*$ (for any choice of coordinates) for all $n > N_1$.
        \item Since $\delta_{\lim} = \bot$, we know that $(p_n) = (\delta_n,\theta_n) \to (\bot,\theta)$. Therefore, there exists an integer $N_2$ for which $n > N_2$ implies $p_n$ is in the a neighborhood of $p_\bot$ where $\pd F \theta$ is within $\epsilon$ of the identity matrix (say for the same choice of coordinates and $\epsilon$) and in particular invertible.
        Therefore, $p_n$ is submersive; since we assumed for contradiction that there are no active submersive points in the fiber, we must conclude that $\pd F\chi(p_n) = \pd F\chi(\delta_n,\theta_n) = 0$. 
        So the first term of \eqref{eq:dH} is zero.
    \end{itemize}
    From these two observations, we deduce that, for all $n > \max(N_1, N_2)$, 
    the quantity $w_n := \pd F \theta(p_n) \pd F \chi (\chi_n,\theta_0)$ 
    on the right side of \eqref{eq:dH},
    is the product of an invertable matrix (whose trace is bounded away from zero) and a vector bounded away from zero, and hence itself a vector $w_n$ bounded away from zero. 
    This forces $v_\chi = 0$. 
    Furthermore, this same line of reasoning applies not only for the points $p_n$ and $p_{n+2}$,
        but for the entire curve they lie on. Parameterizing this curve as a path $\gamma(t)$ along this curve starting at $p_n$ and ending at $p_{n+2}$, we find that the kernel of $dH|_{\gamma(t)}$ has a zero $\chi$-component for all $t$ along this segment.
    Thus the curve $\gamma(t)$ must have zero derivative in its first component ($\chi$), and $\chi$ must be constant along it.        
    And yet $\chi_n < \chi_{n+1} < \chi_{n+2}$ are strictly increasing coordinates! This is a contradiction. 
\end{proof}

\recall{theorem:add-reparam}
\begin{lproof}\label{proof:add-reparam}

    \def\Dir{\mathit{Dir}}
    For each $\theta \in \Theta$, let 
    \[
        \Dir(\theta) := \Big\{ \frac{\partial}{\partial\chi} F(\chi, \theta_0) : \theta_0 \in \Theta, \chi \in \confdom, F(\chi,\theta_0) = \theta \Big\} \subseteq T_{\theta} {\Theta}
    \]
    be the tangent subspace at $\theta$ spanned by derivatives of $F$ at various starting points.
    The key to proving the theorem is to show that 
    the elements of $\Dir(\theta)$ are all parallel and oriented the same direction; 
    this will allow us to use it to define a vector field 
    which locally captures updating with $F$ (up to re-scaling) regardless of the ``original'' starting belief state $\theta_0$. At this point, we can recover an additive representation from the integral curves of this vector field.
    
    Suppose $(\chi_1, \theta_1)$ and $(\chi_2, \theta_2)$ are such that $F(\chi_1,\theta_1) = F(\chi_2,\theta_2) = \theta$. 
    To show that the corresponding directions in $\Dir(\theta)$ are parallel, it suffices to show that the sub-tangent spaces of $T_{\theta} \Theta$ generated by infinitesimal perturbations of $\chi_1$ and $\chi_2$, respectively, are the same.
    For all $\chi_1' > \chi_1$, we know (by \cref{ax:ineq-witness}) that 
    \[
    \exists \tilde \chi_1.~F({\chi_1'}, \theta_1) = F({\tilde\chi_1}, F({\chi_1}, \theta_1)) = F({\tilde\chi_1}, F({\chi_2}, \theta_2)).
    \]
    Thus, for all $\chi_1' > \chi_1$, there exists some $\chi_2' := \tilde \chi_1 \cseq \chi_2 \ge \chi_2$ such that $F(\chi_2', \theta_2) = F(\chi_1', \theta_1)$.
    Symmetrically, for all $\chi_2' > \chi_2$, there exists a corresponding $\chi_1' \ge \chi_1$ with the same property.  
    In particular, this is true for $\chi_1'$ and $\chi_2'$ that are infinitesimally close to $\chi_1$ and $\chi_2$, and thus the ray in the tangent space $T_{\theta}\Theta$ generated by positive perturbations of $\chi_1$ and $\chi_2$ are the same (if nonzero).
    Formally speaking, this argument establishes that either  
    \begin{equation*}
    \begin{aligned}
    \{ d F(v, \theta_1) : v \in T_{\chi_1}\confdom\} = \{ d F(v, \theta_2) : v \in T_{\chi_2}\confdom\},\\
    \text{ or one of the two equals the singleton } \{ \mat 0 \}.
    \end{aligned}
        \label{eq:sametangent}
    \end{equation*}
    (Recall that $T_\chi\confdom$ is the tangent space at $\chi \in \confdom$, and has the same dimension as $\confdom$.)
    It follows that the dimension of $\mathrm{span}(\Dir(\theta))$ is at most the dimension of the confidence domain $\confdom$ itself---and 
    since that domain was assumed to be one-dimensional, we have shown that 
    $\mathrm{dim}\,\mathrm{span}(\Dir(\theta))$ is equal either to one or to zero. 
    Moreover, we have shown that all (nonzero) tangent vectors in $\Dir(\theta)$ point in the same direction.
    
    Define a vector field $X(\theta)$ by a continuous selection from $\Dir(\theta)$ that is nonzero whenever $\Dir(\theta)$. 
    Such a continuous selection exists because $F$ itself is twice continuously differentiable ($C^2$)
    when restricted to $\Theta_\phi$. 
        
    For each point $\theta$:
    if $\Dir(\theta) \ne \{ \mat 0\}$, then select any $(\theta_0, \chi) \in F^{-1}(\theta)$ for which $\frac{\partial}{\partial \chi} F(\theta_0, \chi) \ne 0$. 
    Applying \cref{lem:active-sub}, this guarantees the existence of an active submersion point $\hat p$;
    in turn, by the submersion theorem, this guarantees the existence of a $C^1$ local section
        $\sigma_\theta : U_\theta \to \confdom \times \Theta$
    on some neighborhood $U_\theta \ni \theta$.
    We then define a local vector field on $U_\theta$ according to
        $Y_{\theta}(\theta') := \frac{\partial F}{\partial \chi} (\sigma(\theta'))$. 
    Since $\{ U_\theta \}_{\theta \in \Theta}$ is an open cover of $\Theta$, 
    we know there exists a partition of unity $R = \{ \rho_\theta : U_\theta \to [0,1] \}$ subordinate to it---meaning that this indexed family has the following properties \citep[Thm 2.23]{lee2013smooth}:
    \begin{enumerate}
        \item for all $\theta \in \Theta$, $\rho_\theta(\theta') = 0$ when $\theta' \notin U_\theta$.
        \item every point $\theta' \in \Theta$ has a neighborhood 
            that intersects the support of $\rho_{\theta}$ for only finitely many values of $\theta$.
        \item $\forall \theta' \in \Theta. \sum_\theta \rho_\theta(\theta') = 1$. 
    \end{enumerate} 
    Finally, this allows us to define our vector field as 
    \begin{equation}
        X(\theta') := \sum_{\theta \in \Theta} \rho_\theta(\theta') Y_\theta(\theta).
    \end{equation}
    This is continuous because each $Y_\theta$ is smooth, and only finitely many terms $\rho_\theta$ are nonzero.

    For $\theta \in \Theta$ and any vector field $V \in \mathfrak X(\Theta)$, we use the standard notation $\exp_\theta( V ) := y(1)$ for the unique solution to the differential equation  $\frac{\mathrm dy}{\mathrm dt} = V(y)$ with initial condition $y(0) = y_0$, evaluated at $t=1$.
    By the rescaling lemma  \cite[e.g.,][Lemma 9.3]{lee2013smooth}, $\exp_\theta(t V) = y(t)$ is the result of starting at $\theta$ and following the vector field $V$ for time $t \ge 0$. 
    Since scaling a vector field by a positive scalar field results in the same (or truncated) integral curves after reparameterization, for all $\theta\in \Theta$ and  $\chi \in \confdom$, there exists some $t_{(\theta,\chi)} \in [0,\infty]$ such that $\exp_\theta( t_{(\theta, \chi)} X) = F(\chi,\theta)$.

    With these definitions in place, we define
    $^+\!F(t, \theta) := \exp_\theta(t X)$
    for $t \in [0, \infty]$, 
    and $g(\chi,\theta) := t_{(\theta,\chi)}$.
    \qedhere

\end{lproof}

\commentout{%
\subsection{Additive Representation Theorem}\label{appsec:addrep}

The most difficult math in this paper is the additive representation theorem (\cref{theorem:add-reparam}). We will need a few lemmas.

\begin{lemma}
    If $F : \confdom \times \Theta \to \Theta$ is a commitment function satisfying 
    \cref{ax:combinativity,ax:cont-and-smooth,ax:seq-for-more}
    [L1-L5] and CF, then,
    for all $\theta \in \Theta$ and $\chi_1 \le \chi_2$, 
    there exists a continuously differentiable and strictly monotone path 
    $\gamma : [0,1] \to \confdom$ such that
    $\gamma(0) = \chi_1$, $\gamma(1) = \chi_2$ 
    with the additional property that 
    $\frac{\partial}{\partial s} F(\gamma(s), \theta) \ne 0 $ 
    unless $F(\chi_1, \theta) = F(\chi_2, \theta)$.
\end{lemma}
\begin{lproof}
    \TODO
\end{lproof}
}

\section{Defered Calculations and Further Results}
Beyond the main numbered results of the paper, we have also deferred a few minor calculations to the appendix. 

\textbf{Kalman Combinativity.} 
We claim that pair $(K, r^2)$ forms a confidence domain.
With some simple algebra, one can show that the sequence of updates $(K_2, r_2^2) * (K_1, r_1^2)$ 
is equivalent to a single update with $(K_3, r_3^2)$, where $K_3 = K_1 + K_2 - K_1K_2$ just as in example 1 and the other examples using the $[0,1]$ domian, and
\[
    r^2_3 = \frac{K_2^2 r^2_2 + K_1^2(1-K_2)^2r_1^2}{(K_1+K_2-K_1K_2)^2}.
\]
This is the only non-commutative example we have given.
In the case where $K$ is chosen optimally, this reduces to a single domain with inverse variance combining additively.

\begin{prop} \label{prop:free-additivity}
	If $F: \confdom \times \Theta \to \Theta$ satisfies 
    \cref{ax:zero,ax:cont-and-smooth},
    then there exists another commitment function ${^{::}\!F}$
	(also for beliefs $\Theta$ on observations $\Phi$), that accepts confidences in an extended domain $\confdom' \supseteq \confdom$, has the same behavior as $F$ when restricted to the orginal confidence domain, and in addition satisfies all axioms \cref{ax:zero,ax:cont-and-smooth,ax:seq-for-more,ax:acyclic,ax:combinativity}.
\end{prop}
\begin{lproof}
Consider the new confidence domain
\[
    \Big\{
        \text{finite lists}~[c_1, \ldots, c_n]
		\text{ with each } c_i \in \confdom,
		\quad
        \leqslant
		\quad
		::\;,
		\quad
		[\,]\;,
		\quad
		[\top]\;,
		\quad
        \mathfrak g'
		\,
	\Big\}, \quad\text{where}
\]
\begin{itemize}
\item 
The operation ``$::$'' is list concatenation, except that it collapses instances of $\top$, i.e.,
\[
	[c_1, \ldots c_n] :: [d_1, \ldots, d_m]
	 := \begin{cases}
		 [\top] & \text{ if } \top \in \{c_1, \ldots, c_n,d_1, \ldots,d_m \} \\
		 [c_1, \ldots, c_{n}, d_1, \ldots, d_m] & \text{otherwise.}
 \end{cases}
\]
Concatenating the empty list $[\,]$ on either side has no effect,
by construction, for all $L \in \confdom'$, we have $[\top] :: L = [\top] = L :: [\top]$,
and $::$ is clearly associative, so $\confdom'$ is also a confidence domain.
\item
The order is given by the prefix ordering: $[c_1, \ldots, c_n] \leqslant [d_1, \ldots, d_m]$ iff $n \le m$ 
with $d_i = c_i$ for all $i \in \{0, \ldots, n-1\}$ and $c_i \le d_i$ if $n \ge 1$. 
\item 
The geometry $\mathfrak g'$ is given through the appropriate disjoint sum of product topologies and differentiabl structures, so they are non-interacting discrete components. 
\end{itemize}

The new update rule for this confidence is given by:
	\[
		{^{::}\!F}([c_1, \ldots, c_n], \theta)  :=
				(F^{c_n} \circ \cdots \circ F^{c_1}) (\theta).
	\]
${^{::}\!F}$ has the same behavior as $F$ on the elements that correspond to the original confidence domain, since
$
	{^{::}\!F}(c,\theta) = F(c,\theta),
$
when $c \in \confdom$ is a member of the original domain, 
and it satisfies \cref{ax:combinativity} by construction, since
\begin{align*}
{^{::}\!F}^{[c_1, \ldots, c_n]}_\phi ( ^{::}\!F^{[d_1, \ldots, d_m]}_\phi (\theta) )
		&:=
			F^{d_m}_\phi \circ \cdots \circ F^{d_1}_\phi (
			F^{c_n}_\phi \circ \cdots \circ F^{c_1}_\phi (\theta))\\
		&= (F^{d_m}_\phi \circ \cdots \circ F^{d_1}_\phi \circ
		F^{c_n}_\phi \circ \cdots \circ F^{c_1}_\phi) (\theta) \\
		&= {^{::}\!F}^{[c_1, \ldots, c_n, d_1, \ldots, d_m]}_\phi (\theta) \\
		&= {^{::}\!F}^{[c_1, \ldots, c_n] :: [d_1, \ldots, d_m]}_\phi (\theta).
\end{align*}
Clearly it satisfies \cref{ax:acyclic}.
Finally, for \cref{ax:seq-for-more}, define subtraction either at the final element (if the final element is greater than the number subtracted) or by ablating elements of the list from the right. This satisfies \cref{ax:seq-for-more}. 
\end{lproof}%

\end{document}